\documentclass[conference]{IEEEtran}
\usepackage{times}

\usepackage[numbers]{natbib}
\usepackage{multicol}
 \usepackage[bookmarks=true]{hyperref}

\usepackage{graphicx}
\usepackage{latexsym,ifthen,rotating,calc,textcase,booktabs,color,endnotes}
\usepackage{amsfonts,amssymb,amsbsy,amsmath,amsthm}
\usepackage{multirow}
\usepackage{algorithm}
\usepackage{algpseudocode}
\usepackage{enumerate}
\usepackage{array}
\usepackage{makecell}
\usepackage{subfigure}

\newtheorem{theorem}{Theorem}
\newtheorem{lemma}{Lemma}
\newtheorem{corollary}{Corollary}
\newtheorem{proposition}{Proposition}
\newtheorem{problem}{Problem}
\newcommand{\trans}{T}
\newcommand{\bsym}[1]{\boldsymbol{#1}}
\newcommand{\hbsym}[1]{\hat{\boldsymbol{#1}}}
\newcommand{\norm}[1]{\left\lVert#1\right\rVert}
\newcommand{\pmat}[1]{\begin{pmatrix}#1\end{pmatrix}}
\newcommand{\pbrac}[1]{\left({#1}\right)}
\newcommand{\abs}[1]{\left|{#1}\right|}
\newcommand{\diff}[2]{\frac{\mathrm{d}#1}{\mathrm{d}#2}}
\newcommand{\pdiff}[2]{\frac{\partial #1}{\partial #2}}

\begin{document}

\title{An Efficient Multi-solution Solver for the Inverse Kinematics of 3-Section Constant-Curvature Robots}


\author{
\authorblockN{Ke Qiu, Jingyu Zhang, Danying Sun, Rong Xiong, Haojian Lu, Yue Wang}
\authorblockA{State Key Laboratory of Industrial Control and Technology,\\
Zhejiang University, China}
}



%

\maketitle

\begin{abstract}
Piecewise constant curvature is a popular kinematics framework for continuum robots. Computing the model parameters from the desired end pose, known as the inverse kinematics problem, is fundamental in manipulation, tracking and planning tasks. In this paper, we propose an efficient multi-solution solver to address the inverse kinematics problem of 3-section constant-curvature robots by bridging both the theoretical reduction and numerical correction. We derive analytical conditions to simplify the original problem into a one-dimensional problem. Further, the equivalence of the two problems is formalised. In addition, we introduce an approximation with bounded error so that the one dimension becomes traversable while the remaining parameters analytically solvable. With the theoretical results, the global search and numerical correction are employed to implement the solver. The experiments validate the better efficiency and higher success rate of our solver than the numerical methods when one solution is required, and demonstrate the ability of obtaining multiple solutions with optimal path planning in a space with obstacles.
\end{abstract}

\IEEEpeerreviewmaketitle

\renewcommand{\thefootnote}{}
\footnotetext{Yue Wang and Haojian Lu are corresponding authors.}
\footnotetext{Email: ywang24@zju.edu.cn; luhaojian@zju.edu.cn.}
\renewcommand{\thefootnote}{\arabic{footnote}}

\section{Introduction}

Continuum robots are able to band, elongate or twist when actuated \citep{amanov2021tendon,wang2019geometric}, granting applications in various tasks like bioinspired grasping \citep{li2016progressive} and minimally invasive surgery \citep{alambeigi2020scade,bajo2016hybrid,brij2010ultrasound}. The piecewise constant curvature is a well-known kinematic framework by modelling the robot with concatenated constant-curvature arcs \citep{burgner2015survey}. In manipulation \citep{katzschmann2015autonomous}, tracking \citep{alatorre2022continuum,gonthina2020mechanics}, and planning tasks \citep{duindam2010motionplanning}, the set of model parameters should be computed from the desired end pose to guide the motion, which is the inverse kinematics problem \citep{rus2015design}. Since many continuum robots composed of an inextensible central backbone have fixed section lengths \citep{alatorre2022continuum,lai2022verticalized,ba2021design}, they need at least 3 sections to reach a specified end pose using the piecewise constant curvature model. Nevertheless, the multi-section inverse kinematics remains an open problem.

An analytical method for the inverse kinematics is always preferred because it is fast and able to find all solutions. In \citep{neppalli2009closedform}, a closed-form geometric method is proposed for specified end translations, but it does not consider the tip orientations. Explicit analytical expressions for 2 extensible sections are derived in \citep{garriga2019kinematics}. However, This method cannot be directly generalised to robots with inextensible sections \citep{garriga2019kinematics}.

The inverse kinematics is often formulated as a local optimisation problem over model parameters when analytical expressions are absent, such as the Newton-Raphson method \citep{gonthina2020mechanics} and damped least square method \citep{singh2017performances}. Control schemes can also be implemented to drive the parameters until the tip arrives at the target position \citep{alatorre2022continuum}. However, numerical inverse kinematics is to find only one solution close to the initial value and, additionally, its iteration process depends sensitively on the initial guess, causing inefficient performance \citep{lynch2017modern}.

\begin{figure}[t]
  \centering
  \includegraphics[width = 1.0 \linewidth]{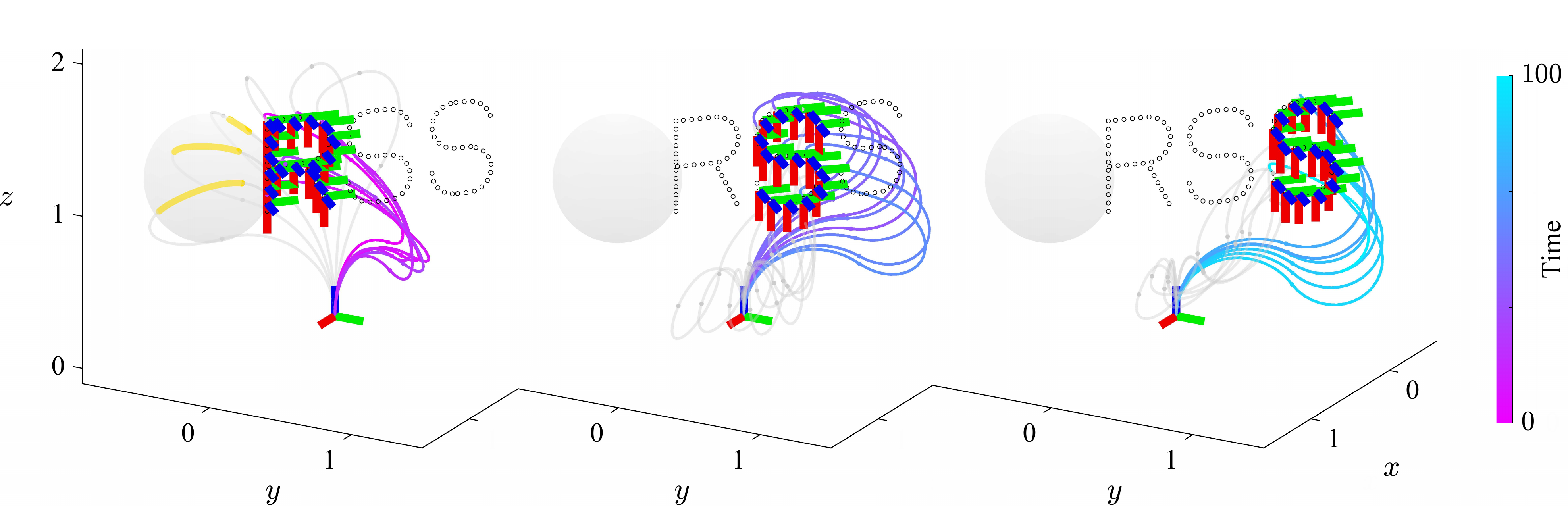}
  \caption{We achieve the optimal RSS-shaped path planning in the parameter space with our multi-solution inverse kinematics solver. The solid line indicates the resultant motion after dynamic programming, while the dashed line indicates the multiple solutions. Collisions are painted yellow.}
  \label{fig_path_rss}
\end{figure}

\begin{figure*}[t]
  \centering
  \includegraphics[width = 0.95 \textwidth]{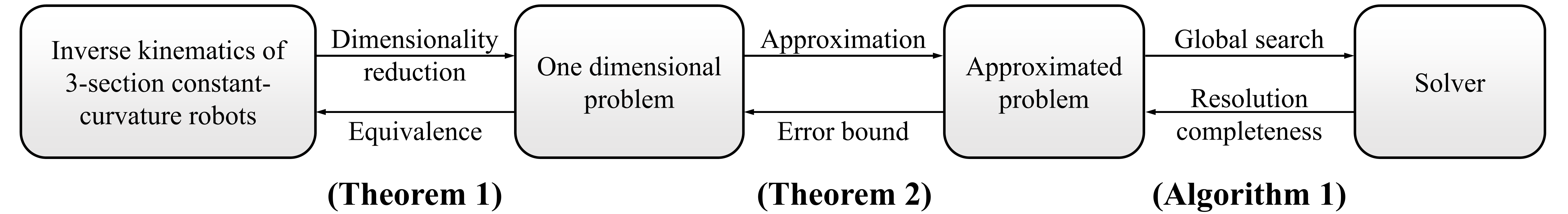}
  \caption{The flowchart of this work. Rightwards arrows are the procedures. Leftwards arrows are the theoretical guarantees.}
  \label{figa}
\end{figure*}

In this paper, we propose an efficient multi-solution solver to address the inverse kinematics problem of 3-section constant-curvature robots. The idea is to simplify the numerical search by theoretical reduction. Specifically, in a theoretical aspect, we derive the analytical conditions to reduce the dimensionality of the 3-section inverse kinematics and finally arrive at an equivalent one-dimensional problem. By further introducing an approximation, the only dimension becomes traversable and the remaining parameters are derived analytically. We show that there exists an error bound of the approximation.
In a numerical aspect, we employ a global search on the one dimension to ensure the resolution completeness of the approximated problem without hurting the efficiency. With a few steps of numerical iterations, we are able to find multiple solutions to the original inverse kinematics problem. The performance in experiments validates the better efficiency and higher success rate of our solver than the pure numerical methods when one solution is required. Besides, the ability to obtain multiple solutions allows us to achieve optimal path planning in a space with obstacles. As shown in Figure~\ref{fig_path_rss}, we solve the configuration at each via point, yielding a path where we later allocate time. In summary, the contributions are presented as follows.

\begin{enumerate}
  \item We propose two analytical conditions for the 2-section robot when the end rotation or translation is specified. We further propose the necessary conditions for the 3-section robot when the end pose is specified.
  \item We reduce the dimensionality of the original inverse kinematics problem from six to one with these conditions. We formalise the one-dimensional problem and show their equivalence.
  \item We introduce an approximation to bring traversability to the reduced problem. An error bound is derived for the approximation. 
  \item We present the solver, which applies a global search and a few steps of iterations. The computational efficiency and multi-solution accessibility are validated by different tasks in experiments. The code is released.\footnotemark
\end{enumerate}
\footnotetext{Homepage: \texttt{https://sites.google.com/view/micsolver}.}

\section{Related Work}\label{related_work}

\textbf{Modelling and Applications.} In order to better manoeuvre a particular continuum manipulator, the strategy employed for modelling the kinematics is essential, concerning the mechanical design, material properties and geometrical characteristics. It is also indispensable for fitting sensor observation \citep{mahoney2016inseparable} or state estimation \citep{lilge2022se3}. Three popular techniques of modelling include (1) The piecewise constant curvature model (see \citet{webster2010design} for a review), (2) The Cosserat rod model \citep{boyer2021dynamics,campisano2021closed,till2020dynamic}, and (3) The finite element model \citep{oliver2022concentric}. Among them, the piecewise constant curvature model is perhaps the most well-known and widely used kinematic framework for continuum robots \citep{burgner2015survey,lilge2022se3,renda2018discrete}. In practice, many tasks such as motion planning \citep{duindam2010motionplanning} and trajectory tracking \citep{alatorre2022continuum,gonthina2020mechanics,lai2022verticalized} require the solutions to the inverse kinematics. The Cosserat rod model and the finite element model solve forward partial differential equations, considering the mechanics nature of robots and usually taking seconds to minutes for an accurate trajectory \citep{renda2018discrete}, while the piecewise constant curvature model is a pure geometric model solving the configuration at each via point, yielding a path where we should later allocate time and interpolate. It is usually computed in milliseconds.

\textbf{Analytical Methods.} A closed-form geometric method is given for multiple sections by treating each section as a straight rigid link together with a spherical joint and then using trigonometric relationships \citep{neppalli2009closedform}. Tip orientations are not accounted for in this method, and approximations are unavoidable when robots have fixed section lengths.
Applying the quaternion description yields a simple formulation of the constant-curvature model, and this leads to explicit analytical expressions for the inverse kinematics \citep{garriga2019kinematics}.
Unfortunately, this work is for constant-curvature continuum robots with 2 extensible sections and cannot be directly generalised to those with inextensible sections \citep{garriga2019kinematics}.

\textbf{Numerical Methods.} It is natural to resort to numerical methods when there are no analytical approaches available. Nonlinear optimisations can be used to find arc parameters that minimise the pose error. The Newton-Raphson method is a common choice with the pseudoinverse of the Jacobian \citep{gonthina2020mechanics,singh2017performances,godage2011novel}. Alternatively, the damped least square method is usually employed for better convergence \citep{singh2017performances}. However, these methods are likely to be influenced by the nonlinearity of the problem, converging to local minima or thrashing about. Multiple initial guesses are often required to get a solution or try to find an alternative solution.
A solution to the problem can also be obtained in the sight of control. Recently, a proportional-integral-derivative control scheme is applied and achieves accurate tracking \citep{alatorre2022continuum}. The steady-state values of parameters are the solution to the inverse kinematics. However, control constants for each parameter need to be tuned individually to account for differences in response and steady-state error caused by coupling effects.
Though numerical methods make it possible to find proper model parameters, they rely on the selection of initial guesses, require plenty of computations, and are unable to find multiple solutions. Therefore, a new solver to handle the inverse kinematics problem of the piecewise constant-curvature model is desired.

\section{Inverse Kinematics of 3-Section Constant-\\Curvature Robots}\label{problem_statement}

\textbf{Notations.} The quaternion algebra $\mathbb{H}$ is the four-dimensional associative algebra over $\mathbb{R}$, and $\mathbb{S}^n$ denotes the $n$-dimensional unit sphere sitting inside $\mathbb{R}^{n+1}$ \citep{hall2013lie}. A unit quaternion $q \in \mathbb{S}^3 \subset \mathbb{H}$ can be written as a sum of real and imaginary parts, i.e.,
\begin{equation}
  \label{qabcd}
  q = a + b \bsym{i} + c \bsym{j} + d \bsym{k},
\end{equation}
where
\begin{equation}
  a, b, c, d \in \mathbb{R}, \quad \sqrt{a^2 + b^2 + c^2 + d^2} = 1.
\end{equation}
The quaternion multiplication is associative,
$$
\begin{aligned}
  \otimes: \mathbb{H} \times \mathbb{H} &\to \mathbb{H}, \\
  q_1, q_2 &\leadsto q_1 \otimes q_2.
\end{aligned}
$$
For a unit quaternion $q$, the multiplicative inverse $q^{-1}$ is identical to its conjugate $q^*$, which is defined by
\begin{equation}
  \label{upstar}
  q^* = a - b \bsym{i} - c \bsym{j} - d \bsym{k}.
\end{equation}
We use $\pbrac{\cdot}^\trans$ to denote the transpose operation on matrices. For a more elaborate elucidation, see Appendix~\ref{appendix_catalog_of_common_formulae}.

\textbf{Forward Kinematics of Piecewise Constant-Curvature Robots.} We follow the conventional assumptions of the piecewise constant curvature model \citep{webster2010design}: (1) sections are modelled as a series of circular arcs, whose bending directions and angles are independent; (2) attachments between two adjacent sections are negligible; (3) adjacent circular arcs are mutually tangential.

Consider a robot modelled in $N$ arc sections. The Cartesian coordinate systems are denoted by $\{F_1\}, \{F_2\}, \dots, \{F_N\}$, respectively, with the origin located at the proximal endpoint of each arc section. There is also a Cartesian coordinate system $\{F_{N+1}\}$ attached to the distal endpoint of the $N$-th arc. For $\lambda = 1, \dots, N$, the relative rotation and translation from $\{F_\lambda\}$ to $\{F_{\lambda+1}\}$ is denoted by a unit quaternion $q_\lambda$ and a vector $^{\lambda+1}_\lambda\bsym{r} = \bsym{r}_\lambda$ with respect to $\{F_\lambda\}$, and section lengths $L_\lambda$ is known. Then the forward kinematics can be written as
\begin{equation}
  \label{sNrot}
  q = q_1 \otimes q_2 \otimes \cdots \otimes q_N,
\end{equation}
\begin{equation}
  \label{sNtrans}
  \bsym{r} = {}^{N+1}_{1}\bsym{r},
\end{equation}
where the translation $\bsym{r}$ is calculated by
\begin{equation}
  \label{sNr}
  {}^{N+1}_{\lambda}\bsym{r} = {}^{\lambda+1}_{\lambda}\bsym{r} + q_\lambda \otimes {}^{N+1}_{\lambda+1}\bsym{r} \otimes q_\lambda^*,
\end{equation}
for $\lambda$ from $N$ to $1$.

\textbf{Model Parameterisation.} Figure~\ref{sketch} presents the parameterisation of a 1-section robot using the constant curvature model, where $\kappa_1$ is the curvature and $\phi_1$ indicates the bending plane. Assume $\kappa_1 L_1 \in [0, \pi]$ and $\phi_1 \in [0, 2 \pi]$ because in most cases a continuum robot is unable to exceed this limitation, and according to the geometric relationships, we obtain
\begin{equation}
  \label{r}
  \bsym{r}_1 = \frac{1}{\kappa_1} \pmat{\cos{\phi_1} - \cos{\kappa_1 L_1} \cos{\phi_1} \\
  \sin{\phi_1} - \cos{\kappa_1 L_1} \sin{\phi_1} \\ \sin{\kappa_1 L_1}}
\end{equation}
Because the rotational axis lies exactly on the $\pbrac{\bsym{i}, \bsym{j}}$-plane,
the element in the last dimension of the corresponding quaternion must be zero.
In fact, we have
\begin{equation}
  \label{q}
  q_1 = \cos{\frac{\kappa_1 L_1}{2}} + \sin{\frac{\kappa_1 L_1}{2}} \pbrac{-\sin{\phi_1} \bsym{i} + \cos{\phi_1} \bsym{j}}.
\end{equation}
Comparing it with~(\ref{qabcd}) we can get
\begin{equation}
  \label{abcd}
  \begin{aligned}
    &a_1 = \cos{\frac{\kappa_1 L_1}{2}}, &b_1& = -\sin{\frac{\kappa_1 L_1}{2}} \sin{\phi_1}, \\
    &c_1 = \sin{\frac{\kappa_1 L_1}{2}} \cos{\phi_1},\quad &d_1& = 0.
  \end{aligned}
\end{equation}
We have $a_1 \ge 0$ since the bending angle $\kappa_1 L_1 \le \pi$. From~(\ref{abcd}) we can see that
\begin{equation}
  \label{prop1pfeq1}
  \kappa_1 = \frac{2}{L_1} \arccos{a_1},\quad \phi_1 = \mathrm{arctan2}{\pbrac{-b_1, c_1}}.
\end{equation}
As a result, the parameters $q_1, \bsym{r}_1$ are equivalent to the widely used arc parameters $\kappa_1, \phi_1$. Furthermore, we note that
\begin{equation}
  \hbsym{r}_1 = \bsym{r}_1 / \norm{\bsym{r}_1} = \pmat{c_1 & -b_1 & a_1}^\trans,
\end{equation}
and
\begin{equation}
  \label{r1norm}
  \norm{\bsym{r}_1} = \frac{L_1 \sqrt{1 - a_1^2}}{\arccos{a_1}}.
\end{equation}
So $q_1$ and $\bsym{r}_1$ can be computed from the three entries of $\hbsym{r}_1$ and a constant length $L_1$. Therefore, $\hbsym{r}_1$ parameterises the 1-section robot \textit{alone}. The same argument shows that if $\hbsym{r}_\lambda$ is known, then both $q_\lambda$ and $\bsym{r}_\lambda$ are uniquely determined. Consequently, we use $\hbsym{r}_\lambda$ as the model parameter.

\begin{figure}[t]
  \centering
  \includegraphics[width=0.8\linewidth]{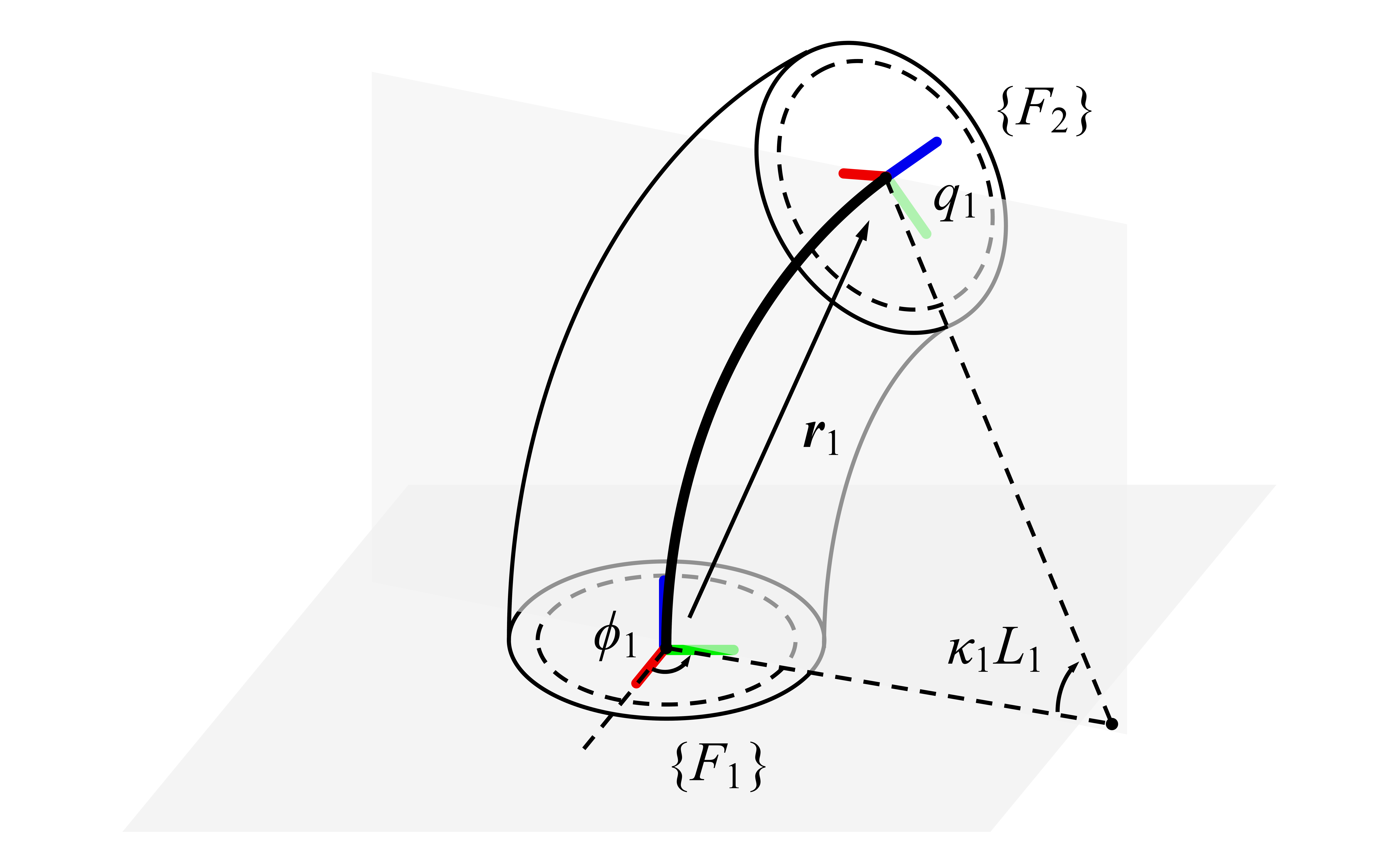}
  \caption{Continuum robots modelled in constant curvature kinematics. The bold line with length $L_1$ indicates the circular backbone curve. The parameters describe the angle $\phi_1$ of the bending plane containing the arc, curvature $\kappa_1$, end rotation $q_1$ and translation $\bsym{r}_1$.}
  \label{sketch}
\end{figure}

\textbf{Problem Statement.} Since $\hbsym{r}_\lambda \in \mathbb{S}^2$, each section offers 2 degrees of freedom. When $N = 2$, whether given an end translation or rotation, the robot will always have one redundant degree of freedom. When $N = 3$, the robot composed of 3 sections offers 6 degrees of freedom at the end. The number matches that of a pose in the task space, hence the robot is able to satisfy both rotational and translational constraints without redundancy. This is the reason that the 3-section model is widely used. Expand~(\ref{sNrot}) and~(\ref{sNtrans}) and we have the problem of the inverse kinematics of 3-section constant-curvature robots.

\begin{problem}[Inverse Kinematics for 3-Section Robots]\label{prob1}
  Given the end rotation $q = \pmat{a & b & c & d}^\trans$ and translation $\bsym{r}$, for $\lambda = 1, 2, 3$, let $\hbsym{r}_\lambda = \pmat{c_\lambda & -b_\lambda & a_\lambda}^\trans$, then $q_\lambda$ and $\bsym{r}_\lambda$ are determined by $\hbsym{r}_\lambda$. Find $\hbsym{r}_1, \hbsym{r}_2, \hbsym{r}_3$ such that
  \begin{equation}
    \label{s3rot}
    q = q_1 \otimes q_2 \otimes q_3,
  \end{equation}
  \begin{equation}
    \label{s3trans}
    \bsym{r} = \bsym{r}_1
    + q_1 \otimes \pbrac{\bsym{r}_2 + q_2 \otimes \bsym{r}_3 \otimes q_2^*} \otimes q_1^*.
  \end{equation}
\end{problem}

\begin{figure}[t]
\centering
\includegraphics[width=0.8\linewidth]{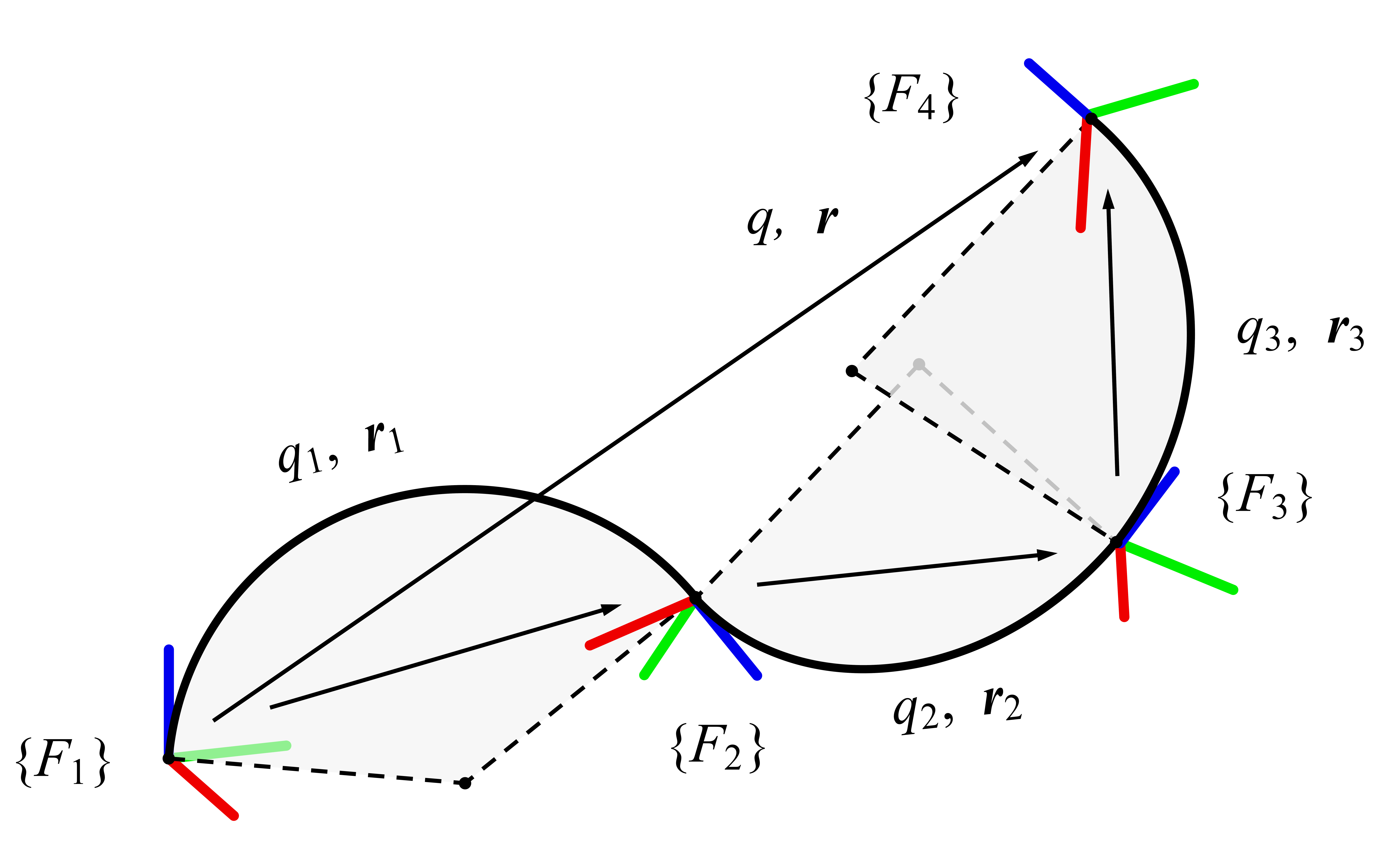}
\caption{An illustration of the variables $q_\lambda, \bsym{r}_\lambda$ and coordinate systems $\{F_\lambda\}$ in a robot composed of 3 circular arcs with a total of six degrees of freedom.}
\label{sketch3}
\end{figure}

\section{Theoretical Results and the Solver}\label{theory_and_the_solver}

As shown in Figure~\ref{figa}, theoretical results are derived so that the range of solutions to Problem~\ref{prob1} can be restricted on a one-dimensional implicit curve (Section~\ref{dimensionality_reduction}). By approximating the curve as a circle on the unit sphere with bounded error, we are able to traverse the circle and solve the remaining variables analytically (Section~\ref{approximation}). A traversal on that circle ensures the resolution completeness. Then numerical correction is employed to arrive at the more accurate solution to Problem~\ref{prob1}. Bridging the analytical and numerical approaches, the algorithm of our solver is accomplished (Section~\ref{the_solver}).


\subsection{Dimensionality Reduction}\label{dimensionality_reduction}
\textbf{2-Section Robots.} We begin from the inverse kinematics of 2-section robots, which is formulated as
\begin{equation}
  \label{s2rot}
  q = q_1 \otimes q_2,
\end{equation}
\begin{equation}
  \label{s2trans}
  \bsym{r} = \bsym{r}_1 + q_1 \otimes \bsym{r}_2 \otimes q_1^*.
\end{equation}
The next lemma gives a property of the 2-section constant curvature inverse kinematics when the end rotation is specified.

\begin{lemma}[Rotational Constraint]\label{lemma1}
  Given an end rotation $q = \pmat{a & b & c & d}^\trans$, the solutions to~(\ref{s2rot}) are those $\hbsym{r}_1 \in \mathbb{S}^2$ satisfying a linear constraint
  \begin{equation}
    \label{lemma1eq1}
    \bsym{n} \cdot \hbsym{r}_1 = 0,
  \end{equation}
  where $\bsym{n} = \pmat{b & c & d}^\trans$.
  When $\hbsym{r}_1$ is specified, $\hbsym{r}_2$ can also be obtained linearly through
  \begin{equation}
    \label{lemma1eq2}
    \hbsym{r}_2 = \bsym{A} \hbsym{r}_1,
  \end{equation}
  where
  \begin{equation}\label{lemma1eq3}
  \bsym{A} = \pmat{-a & -d & c \\ d & -a & -b \\ c & -b & a}.
  \end{equation}
\end{lemma}

\begin{proof}
  See Appendix~\ref{appendix_proof_of_lemma1}.
\end{proof}

Next we analyse the constraint of $\hbsym{r}_1$ when the end translation is given. Following~(\ref{r1norm}) we define
\begin{equation}
  \label{rho}
  \rho\pbrac{a_\lambda, L_\lambda} = \norm{\bsym{r}_\lambda} = \frac{L_\lambda \sqrt{1 - a_\lambda^2}}{\arccos{a_\lambda}},
\end{equation}
then we have the next lemma.

\begin{lemma}[Translational Constraint]\label{lemma2}
  Given an end translation $\bsym{r}$, the solutions to~(\ref{s2trans}) are those $\hbsym{r}_1$ satisfying the constraint
  \begin{equation}
    \label{lemma2eq1}
    \norm{\bsym{v}} = \rho\pbrac{\frac{\bsym{u} \cdot \bsym{v}}{\norm{\bsym{v}}}, L_2},
  \end{equation}
  where
  \begin{equation}
    \label{lemma2eq2}
      \bsym{u} = \pmat{2 a_1 c_1 & -2 a_1 b_1 & 2 a_1^2 - 1}^\trans,
  \end{equation}
  and
  \begin{equation}
    \label{lemma2eq3}
      \bsym{v} = \bsym{r} - \rho\pbrac{a_1, L_1} \hbsym{r}_1.
  \end{equation}
  After choosing an $\hbsym{r}_1$, we can calculate $\hbsym{r}_2$ through
  \begin{equation}
    \label{lemma2eq4}
    \hbsym{r}_2 = \pmat{-1 & 0 & 0 \\ 0 & -1 & 0 \\ 0 & 0 & 1}
    \frac{\pbrac{2 \hbsym{r}_1 \hbsym{r}_1^\trans - \bsym{I}} \bsym{v}}{\norm{\bsym{v}}}.
  \end{equation}
\end{lemma}

\begin{proof}
  See Appendix~\ref{appendix_proof_of_lemma2}.
\end{proof}

\begin{figure}[t]
  \centering
  \subfigure[]{
  \includegraphics[width = 0.4 \linewidth]{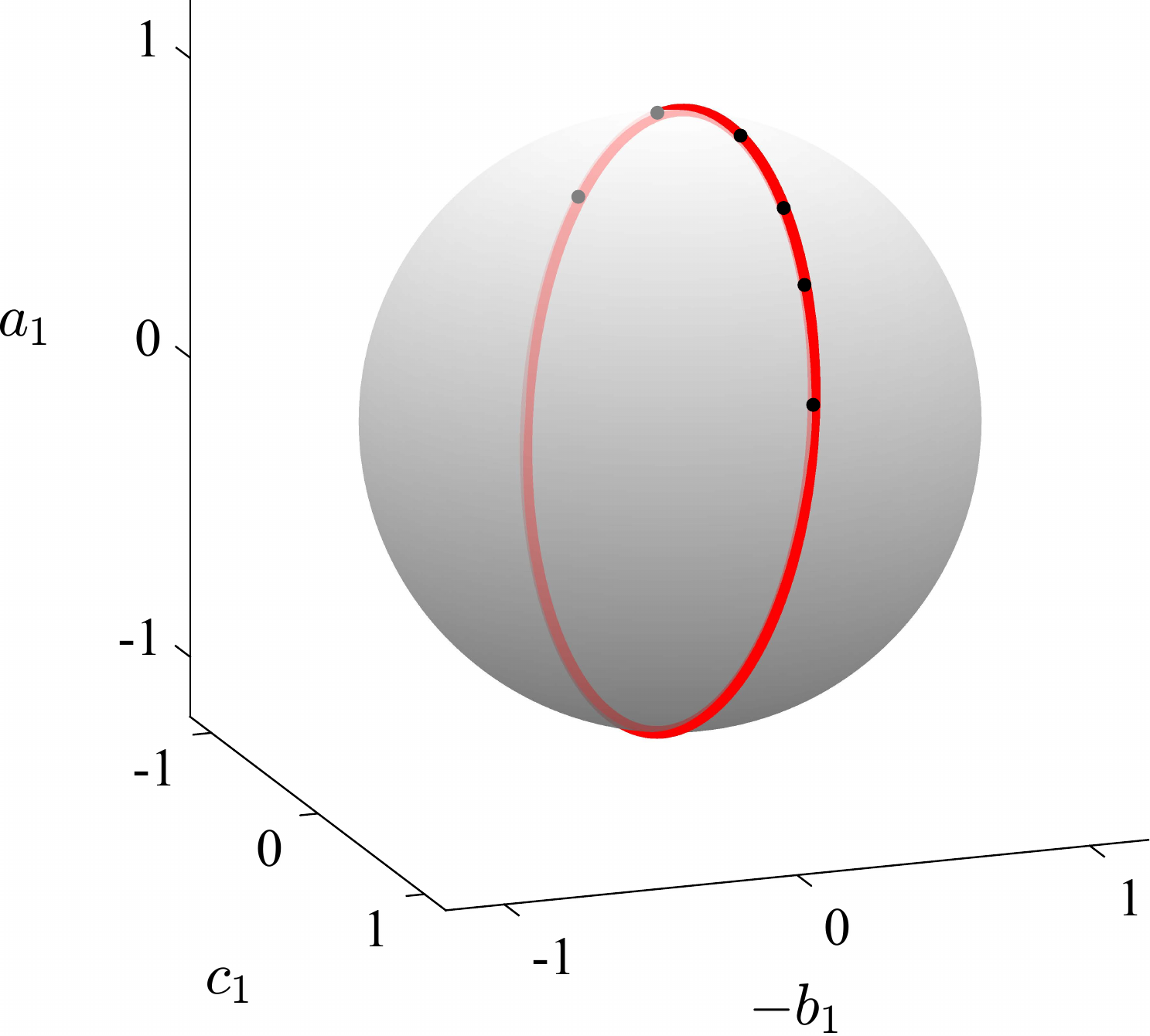}}
  \subfigure[]{
  \includegraphics[width = 0.4 \linewidth]{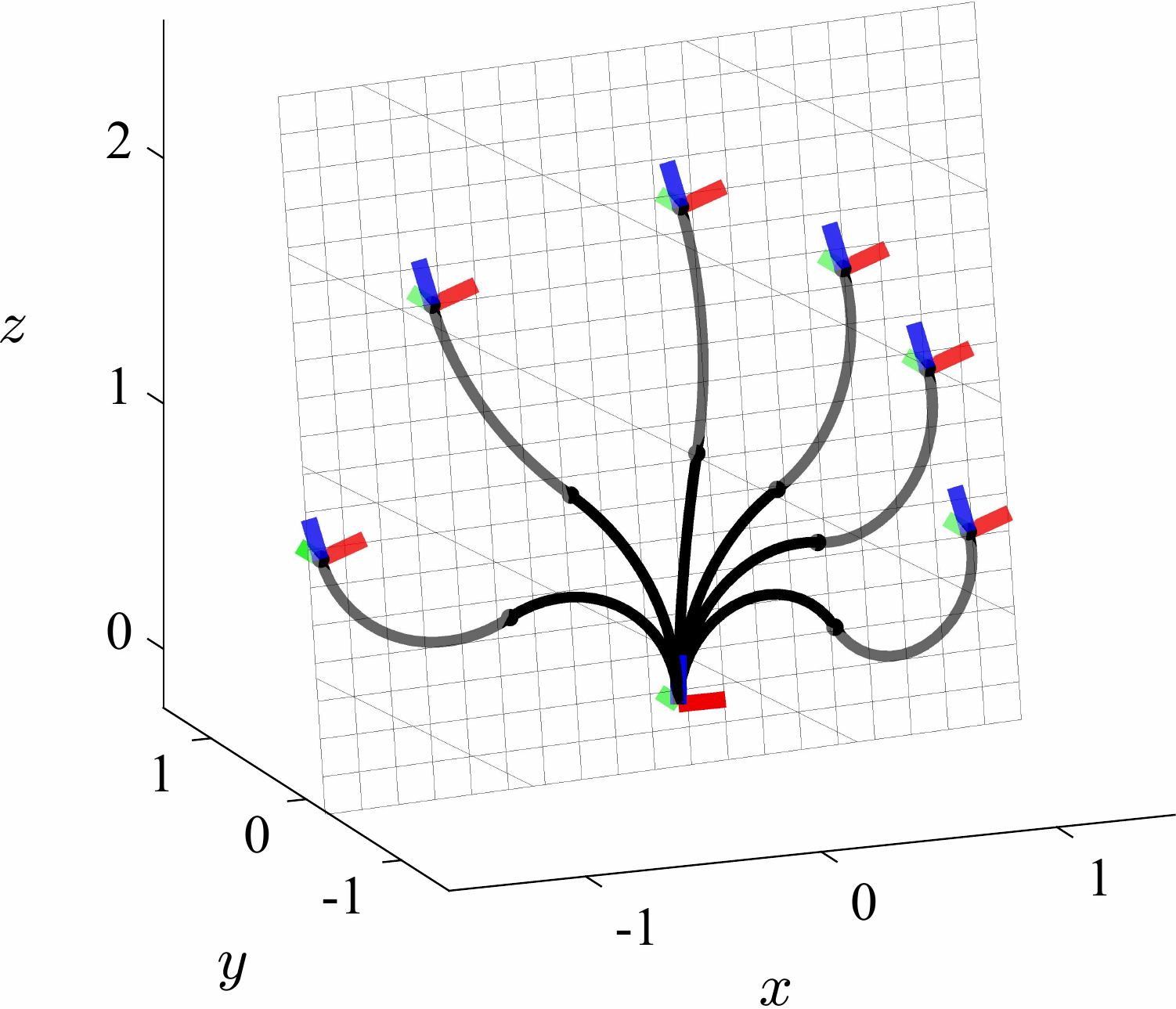}}
  \subfigure[]{
  \includegraphics[width = 0.4 \linewidth]{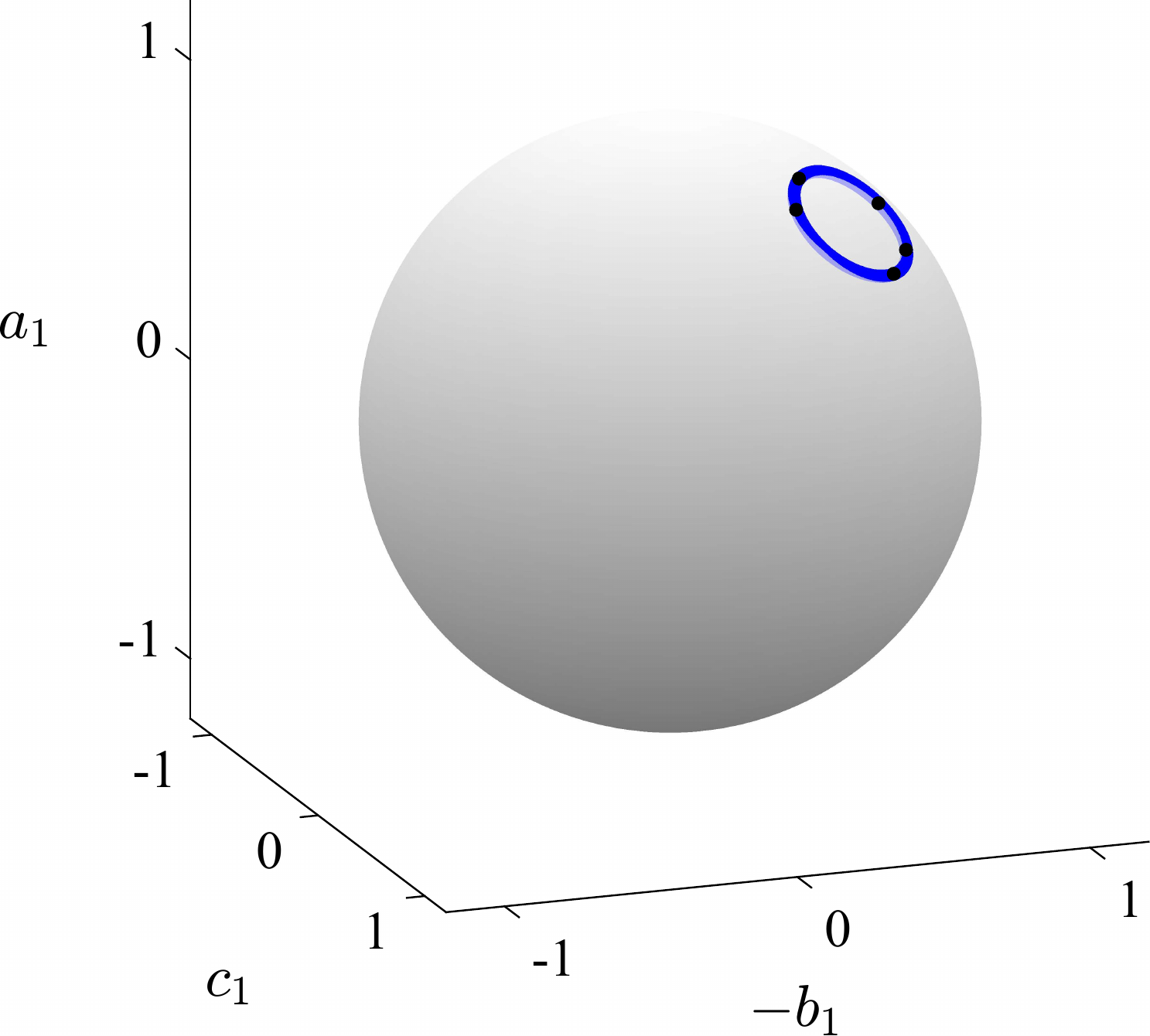}}
  \subfigure[]{
  \includegraphics[width = 0.4 \linewidth]{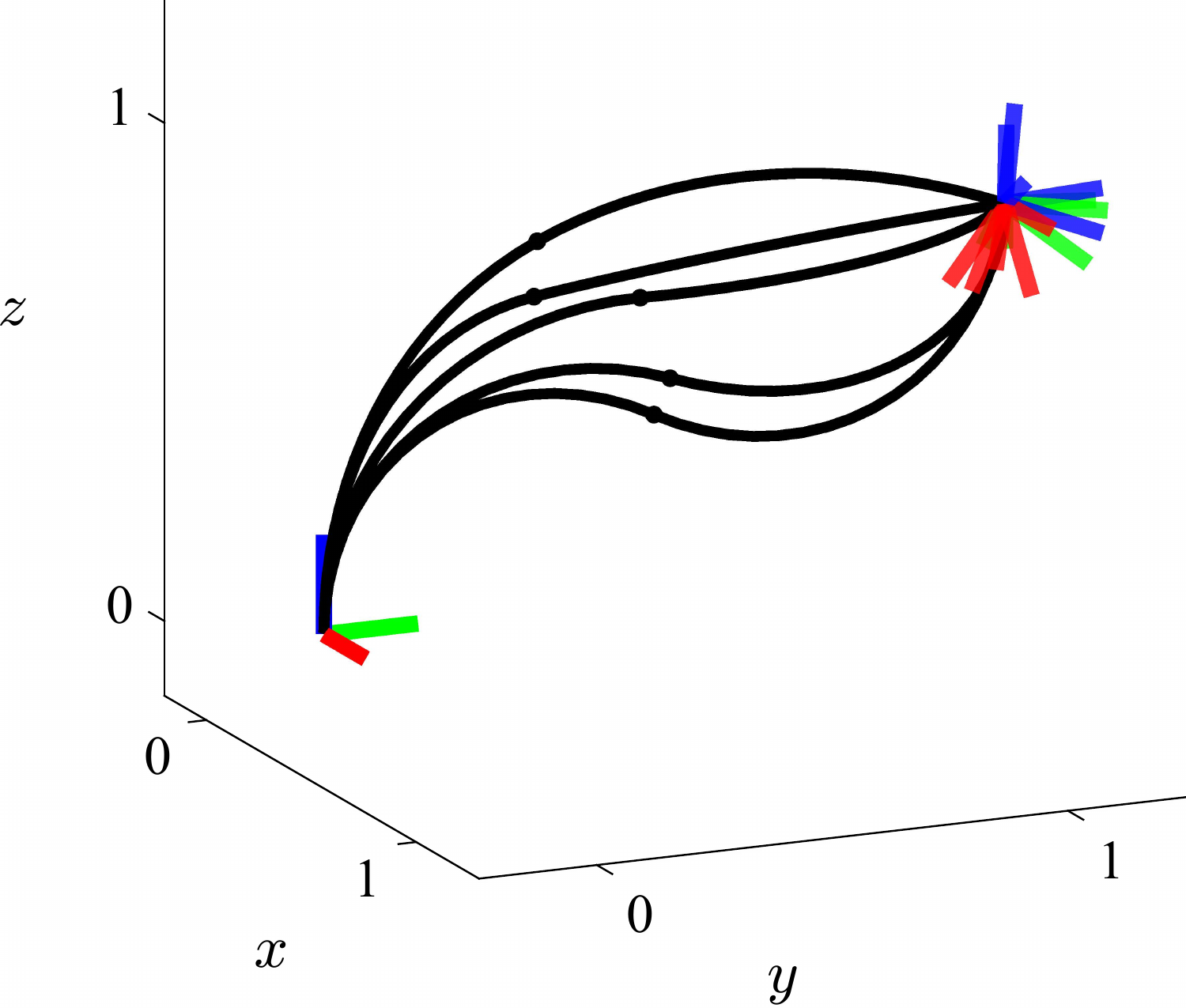}}
  \caption{With selected end rotation $q = \cos\pbrac{\alpha / 2} + \sin\pbrac{\alpha / 2} \bsym{\omega}$, where $\alpha = \pi / 10, \bsym{\omega} = \pbrac{\bsym{i} -4 \sqrt{6} \bsym{j} + \sqrt{3} \bsym{k}} / 10$, we have (a) $\hbsym{r}_1$ satisfies~(\ref{lemma1eq1}), indicated by the red line (Lemma~\ref{lemma1}); (b) for these $\hbsym{r}_1$, end rotations are identical, and the endpoints fall on the plane~(\ref{cor1eq1}) (Corollary~\ref{cor1} in Appendix~\ref{appendix_corollary}). With selected end translation $\bsym{r} = \pmat{1 & 1 & 1}^\trans$, we have (c) $\hbsym{r}_1$ satisfies~(\ref{lemma2eq1}), indicated by the blue line, and (d) for these $\hbsym{r}_1$, end translations are identical (Lemma~\ref{lemma2}). Black dots on the left and robot configurations on the right are in a one-to-one correspondence.}
  \label{fig_rot_trans}
\end{figure}

Note that $\norm{\bsym{u}} = 1$, if we approximate $\rho\pbrac{a_1, L_1} \approx L_1$ and $\rho\pbrac{\bsym{u}^\trans \bsym{v} / \norm{\bsym{v}}, L_2} \approx L_2$, then the constraint~(\ref{lemma2eq1}) becomes the one derived in \citet{neppalli2009closedform}.

\textbf{Illustrative Example.} The behaviour of $\hbsym{r}_1$ that predicted by Lemma~\ref{lemma2}, Corollary~\ref{cor1} (in Appendix~\ref{appendix_corollary}), and Lemma~\ref{lemma3} are illustrated in Figure~\ref{fig_rot_trans}. When fixing an end rotation $q$, the rotational constraint indicates that $\hbsym{r}_1$ lies on a great circle. When fixing an end translation $\bsym{r}$, the translational constraint shows that $\hbsym{r}_1$ lies approximately on a circle.

\textbf{3-Section Robots.} We now move to the 3-section inverse kinematics. Based on our previous results, we find a constraint for both $\hbsym{r}_3$ and $\hbsym{r}_1$, which is formalised in the next lemma.

\begin{lemma}\label{lemma3}
  Given an end rotation $q$ and translation $\bsym{r}$, suppose $\{\hbsym{r}_1, \hbsym{r}_2, \hbsym{r}_3\}$ is a solution to Problem~\ref{prob1}, then for $\lambda = 1, 3$, $\hbsym{r}_\lambda$ satisfies the constraint
  \begin{equation}
    \label{lemma3eq1}
    \bsym{r}^\trans \bsym{B} \hbsym{r}_\lambda - \rho\pbrac{a_\lambda, L_\lambda} d = 0.
  \end{equation}
  where
  \begin{equation}
    \label{lemma3eq2}
    \bsym{B} = \pmat{d & a & b \\ -a & d & c \\ -b & -c & d}.
  \end{equation}
\end{lemma}

\begin{proof}
  See Appendix~\ref{appendix_proof_of_lemma3}.
\end{proof}

From~(\ref{lemma3eq1}) we know that $\hbsym{r}_1$ and $\hbsym{r}_3$ are restricted to \textit{closed} spherical curves with the same form, but no explicit expression has been found yet due to the nonlinearity. We parameterise the curve of $\hbsym{r}_3$ by one variable $t$ formally to help the subsequent analysis,
\begin{equation}
  \hbsym{r}_3 = \hbsym{r}_3\pbrac{t}, t \in [0, 1).
\end{equation}
For any $t$, we specify the end rotation and translation of the first two sections by $\hbsym{r}_3\pbrac{t}$, formulating a 2-section robot inverse kinematics problem,
\begin{equation}
  q_e\pbrac{t} = q \otimes q_3^{-1}\pbrac{t} = q \otimes q_3^*\pbrac{t},
\end{equation}
\begin{equation}
  \bsym{r}_e\pbrac{t} = \bsym{r} - q_e\pbrac{t} \otimes \bsym{r}_3\pbrac{t} \otimes q_e^*\pbrac{t}.
\end{equation}
We have $q_e\pbrac{t} = \pmat{a_e & b_e & c_e & d_e}^\trans$ and we denote
\begin{equation}\label{ne}
  \bsym{n}_e\pbrac{t} = \pmat{b_e & c_e & d_e}^\trans,
\end{equation}
and
\begin{equation}\label{ve}
  \bsym{v}_e\pbrac{t} = \bsym{r}_e\pbrac{t} - \rho\pbrac{a_1, L_1} \hbsym{r}_1.
\end{equation}
We now present the main results of this paper: the statement of the one variable problem and the theorem showing its equivalence to Problem~\ref{prob1}.

\begin{problem}\label{prob2}
  Given the end rotation $q$ and translation $\bsym{r}$, denote $\bsym{u}, \bsym{n}_e, \bsym{v}_e$ as in~(\ref{lemma2eq2}),~(\ref{ne}),~(\ref{ve}), respectively. Let $S$ be the set of $\hbsym{r}_1 \in \mathbb{S}^2$ that satisfies
  \begin{equation}\label{prob2eq1}
    \bsym{n}_e(t) \cdot \hbsym{r}_1 = 0,
  \end{equation}
  \begin{equation}\label{prob2eq2}
    \norm{\bsym{v}_e(t)} = \rho\pbrac{\frac{\bsym{u} \cdot \bsym{v}_e(t)}{\norm{\bsym{v}_e(t)}}, L_2},
  \end{equation}
  \begin{equation}\label{prob2eq3}
    \bsym{A} \hbsym{r}_1 = \pmat{-1 & 0 & 0 \\ 0 & -1 & 0 \\ 0 & 0 & 1}
    \frac{\pbrac{2 \hbsym{r}_1 \hbsym{r}_1^\trans - \bsym{I}} \bsym{v}_e}{\norm{\bsym{v}_e}},
  \end{equation}
  find $t$ that makes $S$ an nonempty set.
\end{problem}

\begin{theorem}\label{theorem1}
  The following statements are equivalent:
  \begin{enumerate}[\rm\bfseries(a)]
    \item Finding $\hbsym{r}_\lambda, \lambda = 1, 2, 3$ in Problem~\ref{prob1}.
    \item Finding the one variable $t$ in Problem~\ref{prob2}.
  \end{enumerate}
\end{theorem}

\begin{proof}
  Given $q, \bsym{r}$. Suppose $\{\hbsym{r}_1, \hbsym{r}_2, \hbsym{r}_3\}$ is the solution to Problem~{\ref{prob1}}, then we can first obtain $t$ from the parametrisation, and according to Lemma~\ref{lemma1},~\ref{lemma2},~\ref{lemma3}, this $\hbsym{r}_1$ satisfies~(\ref{prob2eq1})-(\ref{prob2eq3}) simultaneously. Therefore, $S$ is not empty and $t$ is the solution to Probelm~\ref{prob2}. This shows that \textbf{(a)} implies \textbf{(b)}. Inversely, if we know $t$, then $\hbsym{r}_3$ is specified, and since $S \neq \emptyset$, we take $\hbsym{r}_1 \in S$. From~(\ref{prob2eq1}) and Lemma~\ref{lemma1} we can get $\hbsym{r}_2'$ such that $\{\hbsym{r}_1, \hbsym{r}_2', \hbsym{r}_3\}$ satisfies~(\ref{s3rot}). From~(\ref{prob2eq2}) and Lemma~\ref{lemma2} we can also get $\hbsym{r}_2''$ such that $\{\hbsym{r}_1, \hbsym{r}_2'', \hbsym{r}_3\}$ satisfies~(\ref{s3trans}). We have $\hbsym{r}_2' = \hbsym{r}_2'' := \hbsym{r}_2$ because of~(\ref{prob2eq3}). Therefore $\{\hbsym{r}_1, \hbsym{r}_2, \hbsym{r}_3\}$ is the solution to Problem~\ref{prob1}. This shows that \textbf{(b)} implies \textbf{(a)}.
\end{proof}

\textbf{Illustrative Example.} In Figure~\ref{fig_three_section}, we visualise~(\ref{lemma3eq1}) for both $\hbsym{r}_3$ and $\hbsym{r}_1$ in magenta. When $\hbsym{r}_3$ is at $P_1$, the rotational and translational constraints of $\hbsym{r}_1$ are presented in red and blue, respectively. In this example, the point $P_2$ is the only intersection that satisfies all three constraints. Consequently, we have to check if the condition~(\ref{prob2eq3}) holds for $P_1$ and $P_2$. If it does, then by Theorem~\ref{theorem1} we have found a solution of Problem~\ref{prob1}, namely, the 3-section inverse kinematics problem.

\begin{figure}[t]
  \centering
  \subfigure[]{
  \includegraphics[width = 0.4 \linewidth]{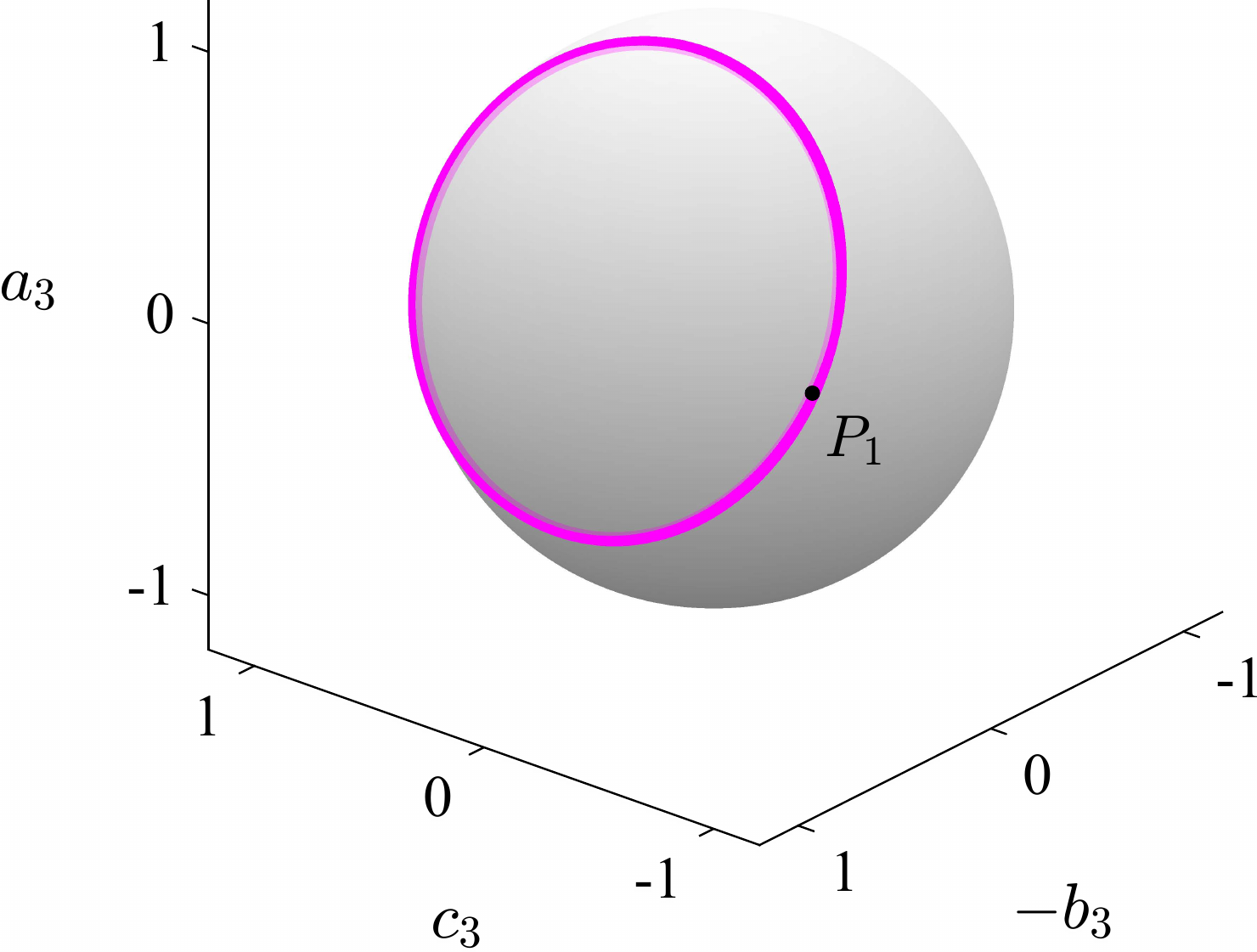}
  \label{fig_three_section_a}}
  \subfigure[]{
  \includegraphics[width = 0.4 \linewidth]{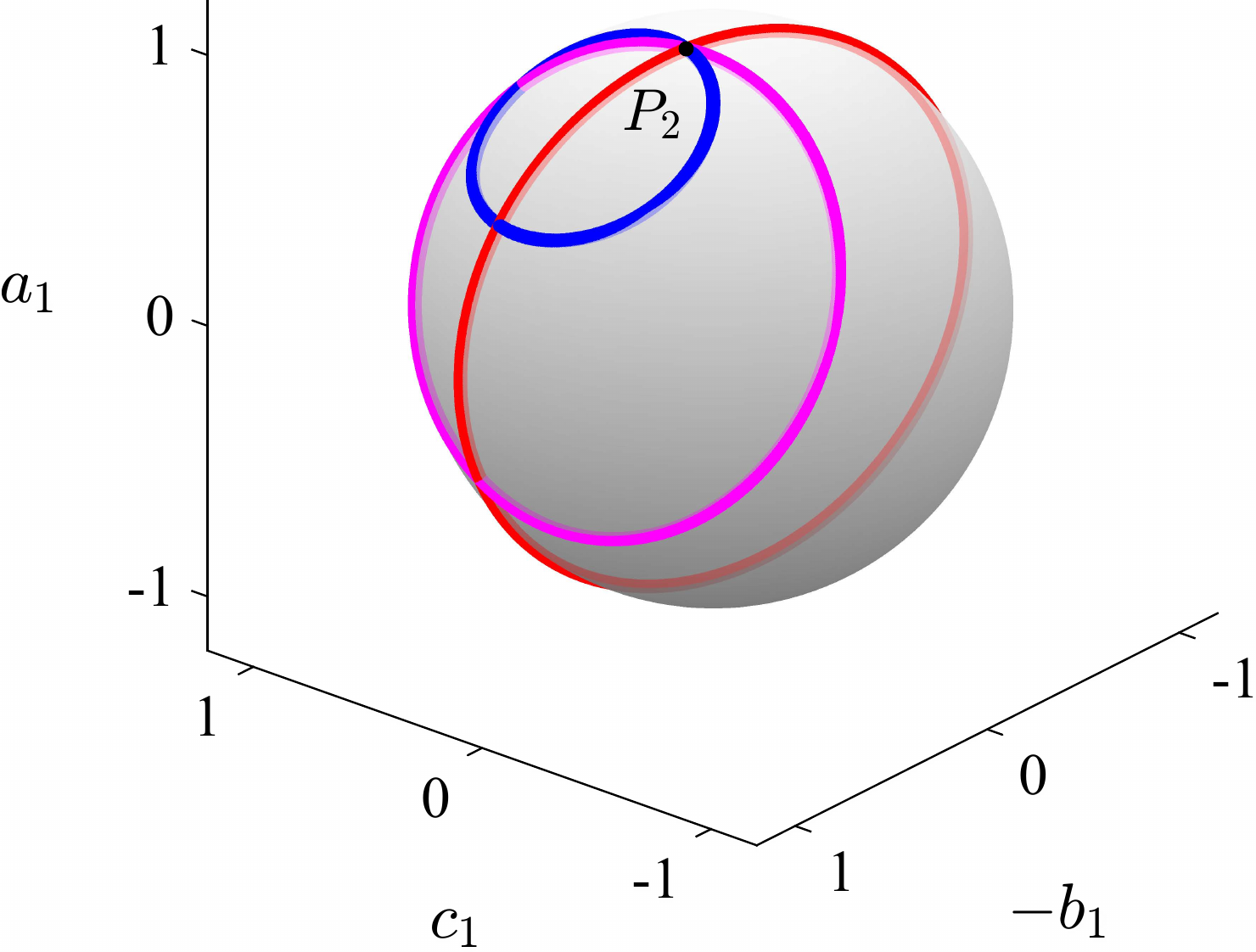}
  \label{fig_three_section_b}}
  \caption{In (a) and (b), the closed spherical curves are visualised in magenta (Lemma~\ref{lemma3}). When $\hbsym{r}_3$ is at the point $P_1$, the rotational constraint of $\hbsym{r}_1$ is indicated with the red line in (b), as well as the translational constraint with the blue line. Three constraints intersect at $P_2$, which is a candidate solution.}
  \label{fig_three_section}
\end{figure}

\subsection{Approximation}\label{approximation}

\textbf{Approximated Constraint of $\hbsym{r}_3$.} The constraint of $\hbsym{r}_3$ is given by~(\ref{lemma3eq1}). Let $\bsym{n}_0 = \bsym{B}^\trans \bsym{r}$, and we approximate~(\ref{lemma3eq1}) as
\begin{equation}\label{r3approx}
  \bsym{n}_0 \cdot \pbrac{\hbsym{r}_3^{(0)} - \gamma L_3 \bsym{r}_0} = 0
\end{equation}
by relaxing the second term in~(\ref{lemma3eq1}) to $\gamma L_3 d$. It refers to a circle on the unit sphere so that $\hbsym{r}_3$ can be analytically traversed by one variable $t$. In~(\ref{r3approx}) the factor $\gamma \in \pbrac{0, 1}$ is a constant determined later to minimise the approximation error throughout the domain $a_3 \in [0, 1]$, and
\begin{equation}
\bsym{r}_0 = \frac{d \bsym{n}_0}{\norm{\bsym{n}_0}^2}.
\end{equation}

\textbf{Approximated Constraint of $\hbsym{r}_1$.} We have known that $\hbsym{r}_1$ satisfies~(\ref{prob2eq1}), which is a linear equation. Since the solution to Problem~\ref{prob1} satisfies~(\ref{lemma3eq1}) according to Lemma~\ref{lemma3}, the solution to Problem~\ref{prob2} must also satisfies~(\ref{lemma3eq1}) by Theorem~\ref{theorem1}, hence we make an approximation analogous to~(\ref{r3approx}) and combine it with~(\ref{prob2eq1}), which leads to the approximation of $\hbsym{r}_1$,
\begin{equation}\label{r1approx}
  \begin{cases}
    \bsym{n}_e \cdot \hbsym{r}_1 = 0,\\
    \bsym{n}_0 \cdot \pbrac{\hbsym{r}_1^{(0)} - \gamma L_1 \bsym{r}_0} = 0.
  \end{cases}
\end{equation}

To evaluate the approximation, we have the following theorem.
\begin{theorem}\label{theorem2}
  The error of approximation~(\ref{r3approx}) is bounded, i.e.,
  \begin{equation}
    \label{r3bound}
    \abs{\hbsym{n}_0 \cdot \pbrac{\hbsym{r}_3^{(0)} - \hbsym{r}_3}} \le \frac{\pbrac{\pi - 2} L_3 \abs{d}}{2 \pi \norm{\bsym{B}^\trans \bsym{r}}},
  \end{equation}
  and when $\norm{\bsym{r}} > L_1$, the error of approximation~(\ref{r1approx}) is also bounded, i.e.,
  \begin{equation}
    \label{r1bound}
    \norm{\hbsym{r}_1^{(0)} - \hbsym{r}_1} \le \arccos{\frac{2L_1\abs{d}}{\pi\norm{\bsym{B}^\trans \bsym{r}}}} - \arccos{\frac{L_1\abs{d}}{\norm{\bsym{B}^\trans \bsym{r}}}}.
  \end{equation}
\end{theorem}

\begin{proof}
  See Appendix~\ref{appendix_proof_of_theorem2}.
\end{proof}

In Figure~\ref{figest}, the black dotted lines show the error bound~(\ref{r3bound}) when we traverse $\hbsym{r}_3$ using~(\ref{r3approx}) instead of~(\ref{lemma3eq1}). From~(\ref{r1approx}) we know that $\hbsym{r}_1$ is the intersection of two circles on the unit sphere. We denote the angle between $\hbsym{r}_1$ and $\hbsym{r}_1^{(0)}$ by $\beta$, then $\max\abs{\beta}$ is an error bound for $\hbsym{r}_1$, geometrically.

\begin{figure}[t]
  \centering
  \subfigure[]{
  \includegraphics[width = 0.4 \linewidth]{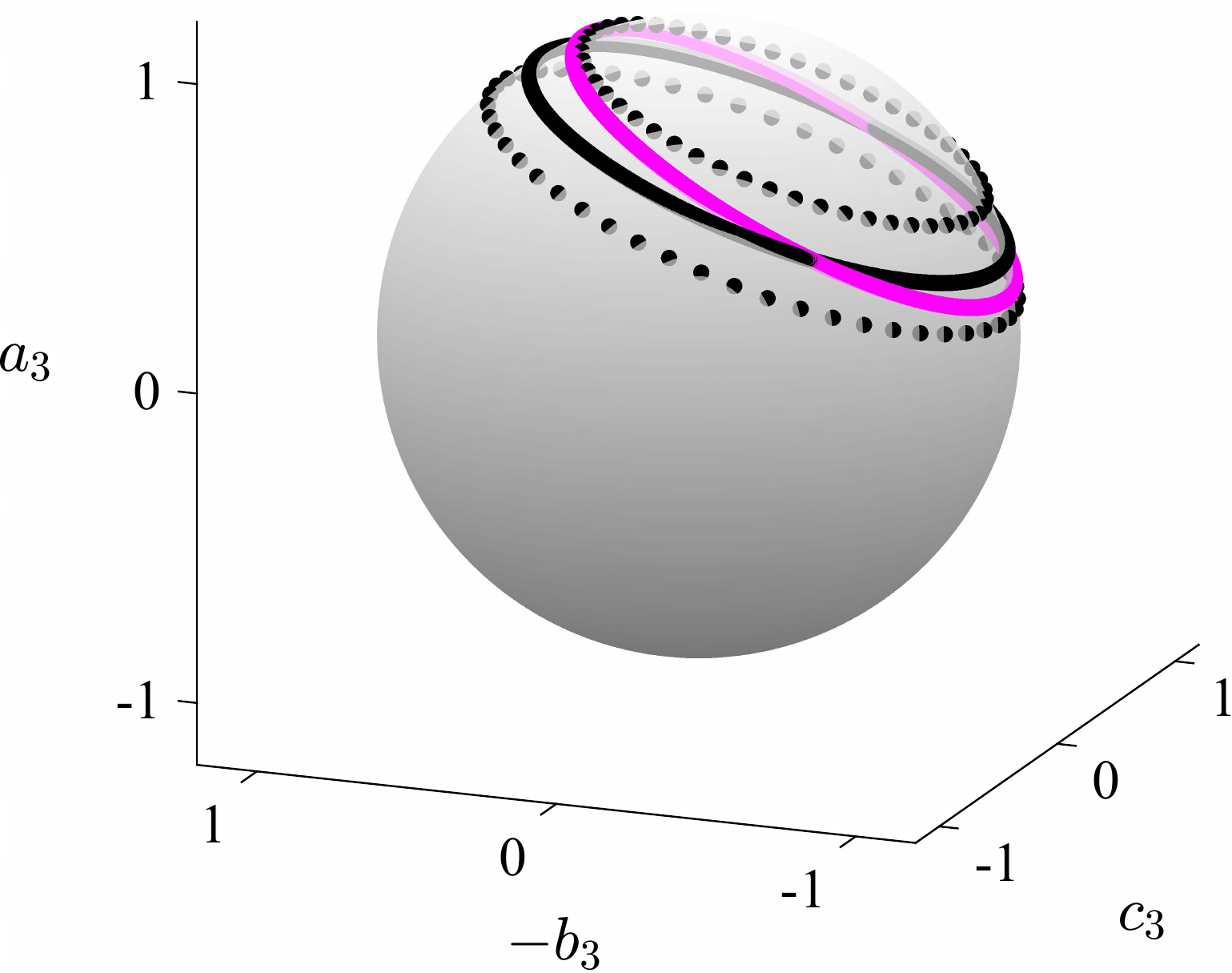}}
  \subfigure[]{
  \includegraphics[width = 0.4 \linewidth]{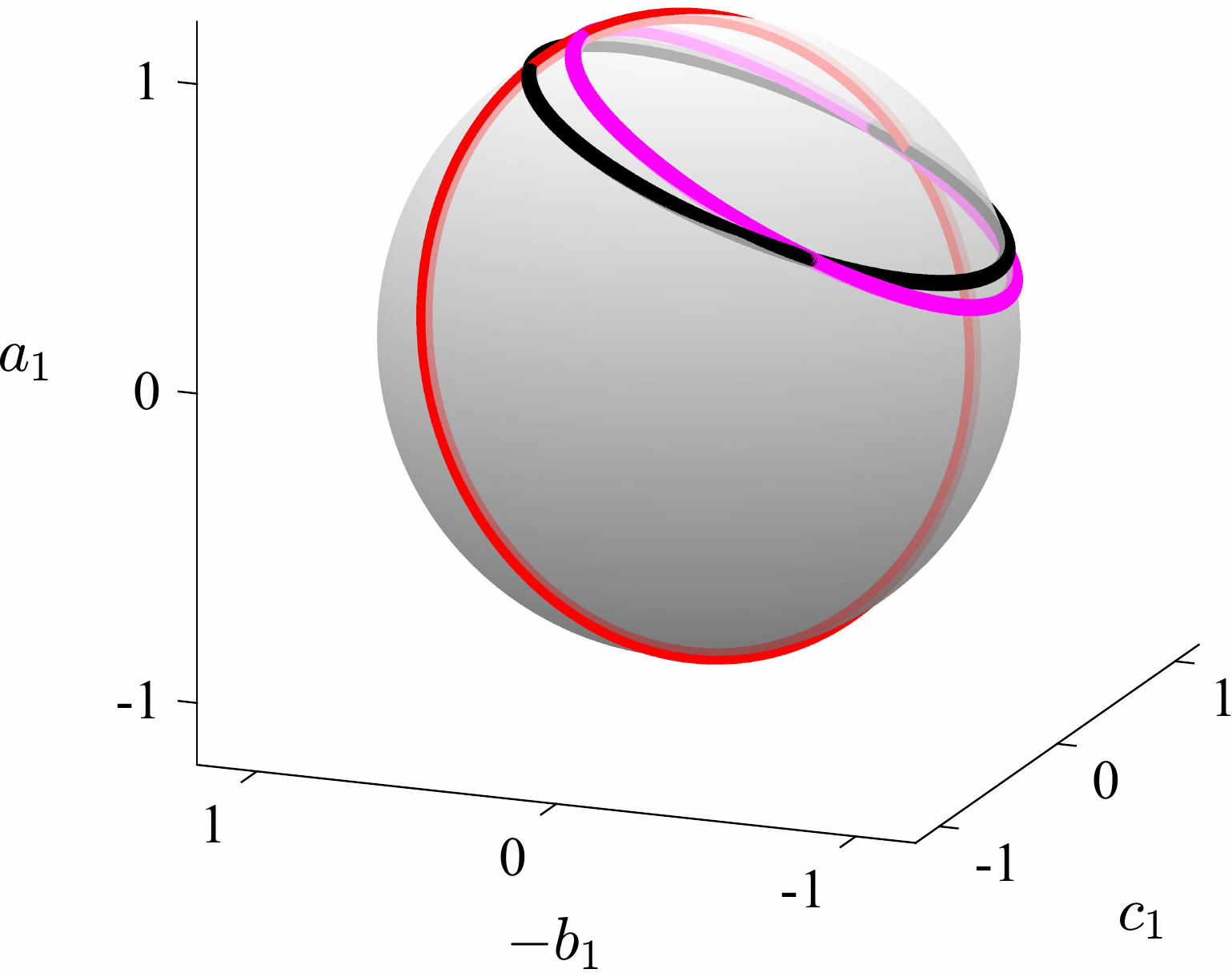}}
  \subfigure[]{
  \includegraphics[width = 0.55 \linewidth]{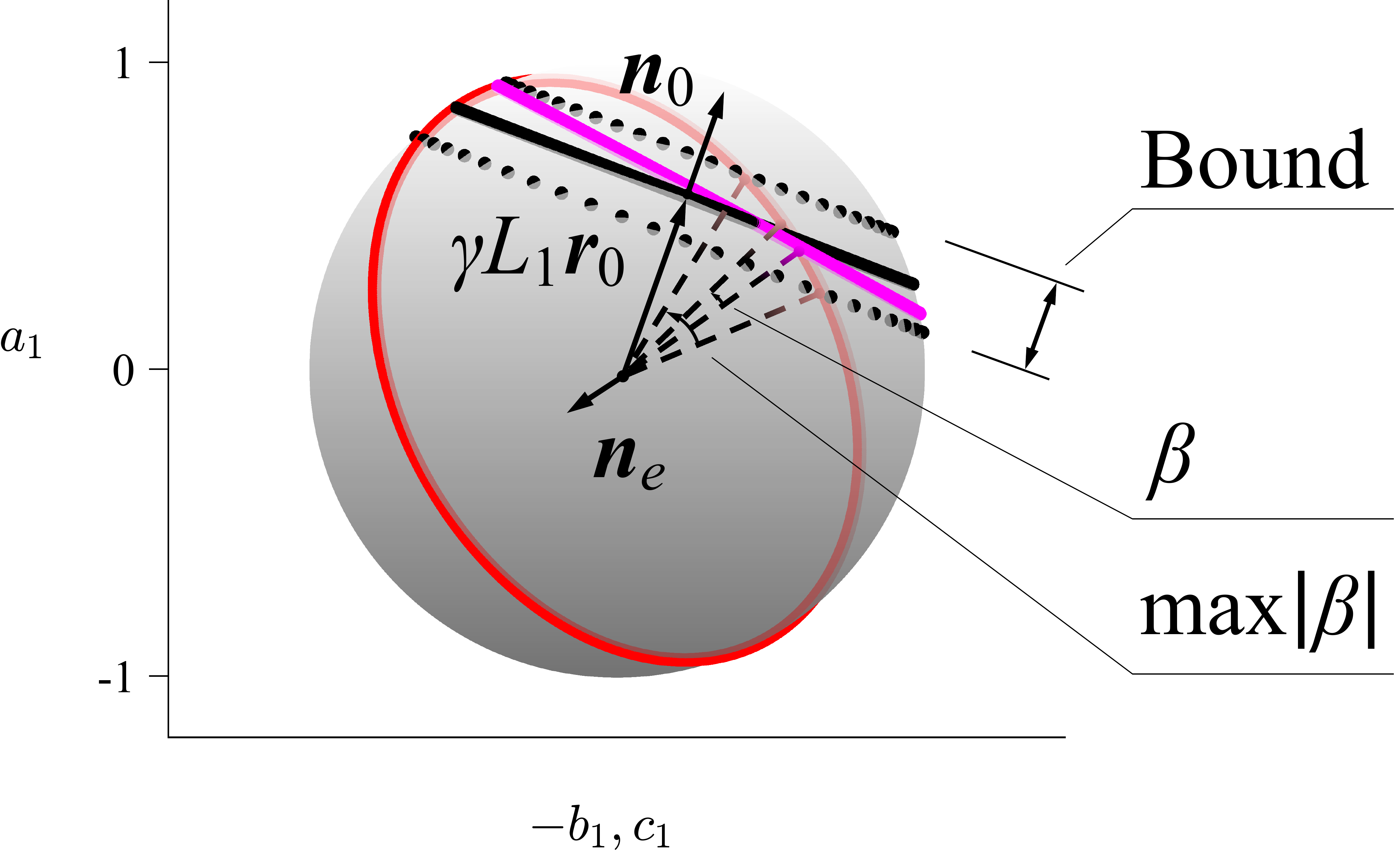}\label{figestr1v}}
  \caption{The black solid lines suggest the approximation of the magenta lines, which represent~(\ref{lemma3eq1}) as previously, while the black dotted lines indicate the error bound. (a) The error is bounded when we traverse $\hbsym{r}_3$ along the black solid line instead of the magenta. (b) We use the intersections of the black and red lines to estimate that of the magenta and red line. (c) Viewing from a line of sight that goes along the direction perpendicular to $\bsym{n}_0$, we see that $\beta$ is the angle between exact $\hbsym{r}_1$ and approximated $\hbsym{r}_1^{(0)}$, so an error bound can be given by $\max\abs{\beta}$.}
  \label{figest}
\end{figure}

\subsection{Inverse Kinematics Solver}\label{the_solver}

\textbf{Circular Traversal.} To make sure resolution completeness, we need to go through all possible discretised $\hbsym{r}_3$. Luckily, the approximation in the previous subsection allows us to make an analytical traversal. It follows from~(\ref{r3approx}) that when $\bsym{n}_0$ is not parallel with the $z$-axis, we construct
\begin{equation}\label{traversal}
  \hbsym{r}_3^{(0)}(t) = \gamma L_3 \bsym{r}_0 + \sqrt{1-\norm{\gamma L_3 \bsym{r}_0}^2} \cdot \bsym{P} \pmat{\sin {2\pi t}\\\cos{2\pi t}\\0},
\end{equation}
where $\hbsym{n}_1 = \hbsym{n}_0 \times \pmat{0&0&1}^\trans$, $\hbsym{n}_2 = \hbsym{n}_0 \times \hbsym{n}_1$, and the orthogonal matrix $\bsym{P} = \pmat{\hbsym{n}_1&\hbsym{n}_2&\hbsym{n}_0}$.

\textbf{Error Evaluation.} For any $t \in [0, 1)$, a unique $\hbsym{r}_3$ is specified, then with~(\ref{r1approx}), $\hbsym{r}_1$ can be found through analytical expressions when $\bsym{n}_0$ is not parallel with $\bsym{n}_e$. We compute $\hbsym{r}_2'$ by Lemma~\ref{lemma1} and $\hbsym{r}_2''$ by Lemma~\ref{lemma2}. At this point we have two sets of model parameters $\{\hbsym{r}_1, \hbsym{r}_2', \hbsym{r}_3\}$ and $\{\hbsym{r}_1, \hbsym{r}_2'', \hbsym{r}_3\}$, we choose the one with a smaller error. To illustrate the calculation of error, we denote the actual end pose built from $\hbsym{r}_\lambda$ by $\bsym{T}$, the desired end pose built from $q, \bsym{r}$ by $\bsym{T}_d$. Define a map from $\hbsym{r}_\lambda \in \mathbb{S}^2$ to $\bsym{\xi}_\lambda \in \mathsf{se}(3)$ as
$$
\begin{aligned}
  \sigma: \mathbb{S}^2 \times \mathbb{S}^2 \times \mathbb{S}^2 &\to \mathsf{se}(3) \times
  \mathsf{se}(3) \times \mathsf{se}(3), \\
  \hbsym{r}_1, \hbsym{r}_2, \hbsym{r}_3 &\leadsto \bsym{\xi}_1, \bsym{\xi}_2, \bsym{\xi}_3,
\end{aligned}
$$
where $\bsym{\xi}_\lambda = L_\lambda \pmat{-\kappa_\lambda \sin{\phi_\lambda} & \kappa_\lambda \cos{\phi_\lambda} & 0 & 0 & 0 & 1}^\trans$ is the twist coordinates \citep{webster2010design}.
From~(\ref{prop1pfeq1}) we know that the map $\sigma$ exists. With the product of exponentials formula \citep{murray1994mathematical}, the end pose $\bsym{T} \in \mathsf{SE}(3)$ has the form
\begin{equation}
  \label{poe}
  \bsym{T} = \bsym{T}\pbrac{\bsym{\xi}_1, \bsym{\xi}_2, \bsym{\xi}_3} = \exp{\bsym{\xi}_1^\wedge}
  \exp{\bsym{\xi}_2^\wedge} \exp{\bsym{\xi}_3^\wedge}.
\end{equation}
We evaluate the error regarding the norm of the body twist from $\bsym{T}$ to $\bsym{T}_d$, i.e.,
\begin{equation}
  \label{err}
  e = e\pbrac{\bsym{\xi}_1, \bsym{\xi}_2, \bsym{\xi}_3}
  = \norm{\pbrac{\log{\bsym{T}^{-1} \bsym{T}_d}}^\vee}.
\end{equation}

\textbf{Candidate Selection.} By traversing $t$ in $[0, 1)$, we have a series of model parameter $\{\hbsym{r}_1, \hbsym{r}_2, \hbsym{r}_3\}$, for which we can build an error function $e\pbrac{t}$. Suppose $\bar{t}$ is a solution to Problem~\ref{prob2}, then the derivative of the error function at $\bar{t}$ must be zero, i.e., $\left.\diff{}{t}\right\vert_{t = \bar{t}}e\pbrac{t} = 0$. Therefore, when $t$ is being traversed, we can collect all the local minima by checking
\begin{equation}
  e\pbrac{\bar{t}} < e\pbrac{\bar{t} + \Delta t},\quad e\pbrac{\bar{t}} \ge e\pbrac{\bar{t} - \Delta t}.
\end{equation}
We calculate the model parameters $\{\hbsym{r}_1, \hbsym{r}_2, \hbsym{r}_3\}$ from $\bar{t}$ as candidate solutions to Problem~\ref{prob1}.

\textbf{Numerical Correction.} The exponential coordinates of the candidates are regarded as the initial values for the optimisation problem below:
\begin{equation}
  \begin{aligned}
    &\mathop{\mathrm{minimise}}\limits_{\bsym{\xi}_1, \bsym{\xi}_2, \bsym{\xi}_3}\quad \norm{\pbrac{\log{\bsym{T}\pbrac{\bsym{\xi}_1, \bsym{\xi}_2, \bsym{\xi}_3}^{-1} \bsym{T}_d}}^\vee}^2 \\
    &\mathrm{subject~to}\quad \bsym{\xi}_1, \bsym{\xi}_2, \bsym{\xi}_3 \in \mathsf{se}(3).
  \end{aligned}
\end{equation}  
We employ the Newton-Raphson or damped least square method to reach the accurate solution of Problem~\ref{prob1}. Filtering the candidates unable to converge to the accuracy threshold, we form successful ones as our final result.

Because of the error bound in Theorem~\ref{theorem2}, the initial values are actually very close to the real solutions, therefore only a few steps of iterations are required in the numerical correction. In summary, we present Algorithm~\ref{alg3}.

\begin{algorithm}
  \caption{Inverse Kinematics Solver.}\label{alg3}
  \begin{algorithmic}[1]
  \Require $q \in \mathbb{S}^3, \bsym{r} \in \mathbb{R}^3, L_1, L_2, L_3 > 0$
  \Procedure{Solver}{$q, \bsym{r}$}
    \State $\Delta t \gets$ A search resolution
    \For{$t = 0, t \gets t + \Delta t, t < 1$}
    \State $\hbsym{r}_3, \hbsym{r}_1 \gets \hbsym{r}_3^{(0)}\pbrac{t}, \hbsym{r}_1^{(0)}\pbrac{t}$ \Comment{(\ref{traversal}),~(\ref{r1approx})}
    \State $\hbsym{r}_2', \hbsym{r}_2'' \gets$ Compute~(\ref{lemma1eq2}) and~(\ref{lemma2eq4})
    \State $e, \hbsym{r}_2 \gets$ min\{$e(\hbsym{r}_2'), e(\hbsym{r}_2'')$\} \Comment{(\ref{err})}
    \If{$t$ is a local minimum}
    \State Collect $\{\hbsym{r}_1\pbrac{t}, \hbsym{r}_2\pbrac{t}, \hbsym{r}_3\pbrac{t}\}$ in $S_c$.
    \EndIf
    \EndFor
    \For{each elements in $S_c$}
    \If{numerical correction converges}
    \State \textbf{output} $\hbsym{r}_1, \hbsym{r}_2, \hbsym{r}_3$
    \EndIf
    \EndFor
    \If{no solutions are found}
    \State \textbf{goto} line 3 with $\Delta t$ a finer search resolution
    \EndIf
  \EndProcedure
  \end{algorithmic}
\end{algorithm}

\section{Simulation Results}\label{experimental_results}

We showcase our 3-section inverse kinematics solver in two scenarios: one collision-free solution for a randomly given pose in a free task space or with obstacles (Section~\ref{one_solution},~\ref{an_example}), and path planning under the given sequence of via points (Section~\ref{path_planning}).

\subsection{Setup}
The robot used in our experiments is composed of 3 sections with an identical length 1. The evaluation of benchmark algorithms and our solver contains the following steps: (1) Sample random arc parameters $\{\pbrac{\kappa_\lambda, \phi_\lambda}: \lambda = 1, 2, 3.\}$ with uniform distributions and compute $\bsym{\xi}_1$, $\bsym{\xi}_2$, $\bsym{\xi}_3$. (2) Generate the desired pose $\bsym{T}_d$ using the forward kinematics~(\ref{poe}), then compute $q$, $\bsym{r}$ from $\bsym{T}_d$. (3) Solve the inverse kinematics problem with different algorithms. (4) Repeat the above steps for 2000 times and evaluate the statistics on the results.

For the numerical correction, the criterion of success is that the error~(\ref{err}) is less than 0.01 and the total number of iterations does not exceed the maximum allowed value.

Since the overall average runtime is strongly influenced by the unsuccessful runs that require the maximum number of iterations, we pick out the successful portion to provide a more qualitative indicator of the performance for comparison.

The code for simulations is implemented in MATLAB R2022b and carried on a laptop computer with a 2.10 GHz Intel Core i7-1260P processor.

\subsection{Benchmarks}

As mentioned, solving the inverse kinematics problem always relies on numerical methods. It is essentially a root-finding problem for quaternion equations~(\ref{s3rot}) and~(\ref{s3trans}). Especially, among various root-locating methods, we select several typically representative ones. The Newton-Raphson method is commonly used \citep{gonthina2020mechanics,singh2017performances,godage2011novel}, so we choose it as one of the benchmarks. Additionally, we implement the gradient descending algorithm for this problem. A built-in MATLAB optimiser \texttt{fminsearch}, which is a derivative-free method using the Nelder-Mead simplex algorithm \citep{lagarias1998convergence}, is also adopted as a benchmark.

Our solver does not require an initial guess, while for benchmarks we provide random arc parameters as the initial guess because the straight configuration is essentially a singularity \citep{santina2020improved}. The Jacobian matrix in the iterative procedures has been derived in \citep{chirikjian2011stochastic}. We put the expressions in Appendix~\ref{appendix_jacobian}.

\subsection{One Solution in the Task Space}\label{one_solution}

\textbf{Free Space.} We start from a free task space. Since the target pose is computed from the forward kinematics with random model parameters, the inverse kinematics problem is ensured to be solvable. Different methods may find either the parameters generating the desired pose, or an alternative solution. Regardless of which one, all the methods stop as soon as one solution is found. The result is shown in Table~\ref{tab_free_space}.

\textbf{With Obstacles.} Structured obstacles are introduced in this scenario. The obstacles are spheres with a diameter of 0.4, and their centre points form a square lattice in the task space with 0.8 in $x$- and $y$-axis, 1.0 in $z$-axis. In each repeat, we generate parameters whose corresponding configuration should be collision-free at the same time, so every inverse kinematics problem is still solvable here. The methods stop as soon as one arbitrary collision-free solution is found. For pure numerical methods that execute only once, the success rate would definitely decrease due to the influence of obstacles. So the Newton-Raphson method is allowed to execute for at most 5 times, labelled as Newton-Raphson (5). The result is presented in Table~\ref{tab_with_obstacles}.

\textbf{Comparison on Efficiency.} We judge the computational efficiency by the runtime. When there are no obstacles, our solver takes about 15\% runtime relative to the Newton-Raphson method, which is the best among all benchmarks. Our solver performs additional computations before starting the numerical correction, which is not needed in benchmark algorithms. Yet the result shows that the operations benefit the overall efficiency because the provided initial value decreases plenty of iterations in the following numerical correction.

When obstacles exist, less than 19\% runtime is achieved relative to the Newton-Raphson (5) method. We note that in this experiment our solver takes almost double the time. The increase in runtime is caused by two reasons. On the one hand, when a solution satisfying the desired pose but colliding with obstacles is found, the numerical correction is applied to the next candidate in order to find an alternative collision-free solution. This proportion is 9.35\%. On the other hand, when no solutions are found, the solver would employ a retraversal with a finer resolution. The proportion of retraversal is 0.50\% in free space, and increases to 2.90\% when considering obstacles.

Notably, with our solver, we observe that some initial values are directly output as solutions \textit{without} any iterations. The proportions are 17.55\% and 21.75\% in the two experiments, respectively. This strongly reduces the computations and meanwhile validates our theoretical results.

\textbf{Comparison on Success Rate.} The success rate reflects the dependency of a numerical algorithm on initial values and the ability to obtain multiple solutions when obstacles exist. The derivative-free method and gradient method are more likely to be trapped by local minima. The Newton-Raphson method is sometimes able to escape from local minima yet is more likely to thrash about until running out of all iterations.

Our solver maintains the full percentage of success rate in both experiments, which owes to the efficient multi-solution finding. The results validate the resolution completeness for the approximated problem, as well as the error bound that allows the numerical correction to converge in a few steps inside the bound.

\begin{table}[t]
  \centering
  \caption{Performance of Different Methods for Solving the Inverse Kinematics Problem.\label{tab_free_space}}
  \begin{tabular}{p{0.3\linewidth}p{0.18\linewidth}p{0.18\linewidth}p{0.14\linewidth}}
  \toprule
  Method & \multirow{2}{*}{\makecell[tl]{Success\\Rate (\%)}} & \multicolumn{2}{c}{Runtime (ms)}\\
  \cmidrule(r){3-4}
  & & Successful & Total\\
  \midrule
Derivative-free 	 & 47.10 	 	 	 & 59.44 	 	 	 & 68.16\\
Gradient 	 	 & 51.05 	 	 	 & 95.23 	 	 	 & 178.41\\
Newton-Raphson 	 & 94.20 	 	 	 & 10.54 	 	 	 & 12.18\\
Our Solver 	 	 & \textbf{100.00} 	 & \textbf{1.59} 	 & \textbf{1.59}\\
  \bottomrule
  \end{tabular}
\end{table}

\begin{table}[t]
  \centering
  \caption{Performance of Different Methods for Solving the Inverse Kinematics Problem with Lattice Obstacles.\label{tab_with_obstacles}}
  \begin{tabular}{p{0.3\linewidth}p{0.18\linewidth}p{0.18\linewidth}p{0.14\linewidth}}
  \toprule
  Method    & \multirow{2}{*}{\makecell[tl]{Success\\Rate (\%)}} & \multicolumn{2}{c}{Runtime (ms)}\\\cmidrule(r){3-4}
  & & Successful        & Total\\
  \midrule
Derivative-free 	 & 40.00 	 	 	 & 57.28 	 	 	 & 67.40\\
Gradient       	 & 45.65 	 	 	 & 95.22 	 	 	 & 181.11\\
Newton-Raphson     	 & 82.75 	 	 	 & 10.83 	 	 	 & 12.09\\
Newton-Raphson (5) 	 & 98.75 	 	 	 & 16.15 	 	 	 & 17.00\\
Our Solver 	 	 & \textbf{100.00} 	 & \textbf{3.05} 	 & \textbf{3.05}\\
  \bottomrule
  \end{tabular}
\end{table}

\subsection{An Example of Multiple Solutions}\label{an_example}

We present an example of multiple solution finding of our solver in a space with obstacles. For a selected end rotation $q$ and translation $\bsym{r}$, Figure~\ref{fig_4_solutions_err} shows the error curve $e\pbrac{t}$, $t \in [0, 1)$ obtained from traversing $\hbsym{r}_3$ with a search resolution $\Delta t = 0.01$. We box the local minima as candidates for the following corrections. Finally, a total of 4 candidates all converge in less than 2 steps of iterations. As a result, our solver gets 4 multiple solutions for the desired end pose, as presented in Figure~\ref{fig_4_solutions_1}-\ref{fig_4_solutions_4}.

Interestingly, 2 to 4 solutions are the most frequently observed results during our experiments. We speculate that the number of solutions of the 3-section inverse kinematics problem belongs to this range.

\begin{figure}[t]
  \centering
  \subfigure[]{
  \includegraphics[width = 0.88 \linewidth]{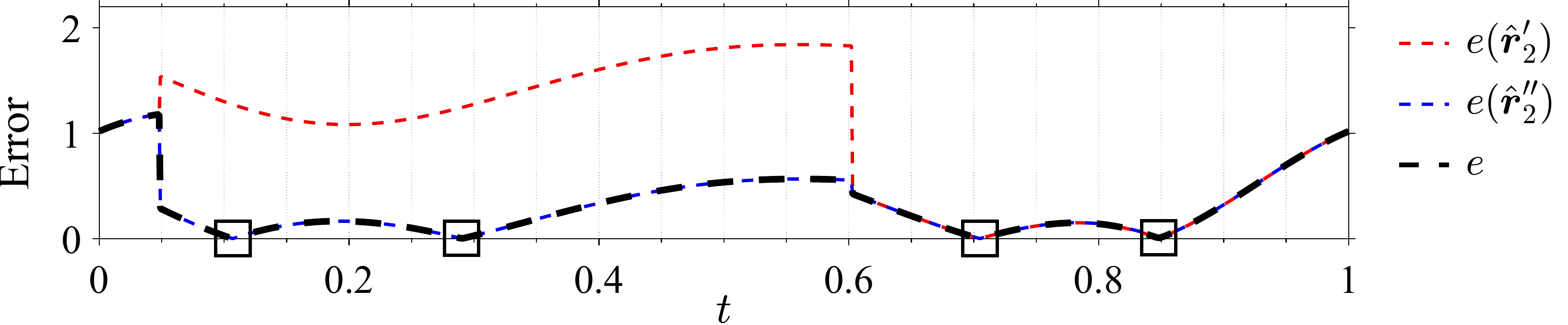}\label{fig_4_solutions_err}}
  \subfigure[]{
  \includegraphics[width = 0.42 \linewidth]{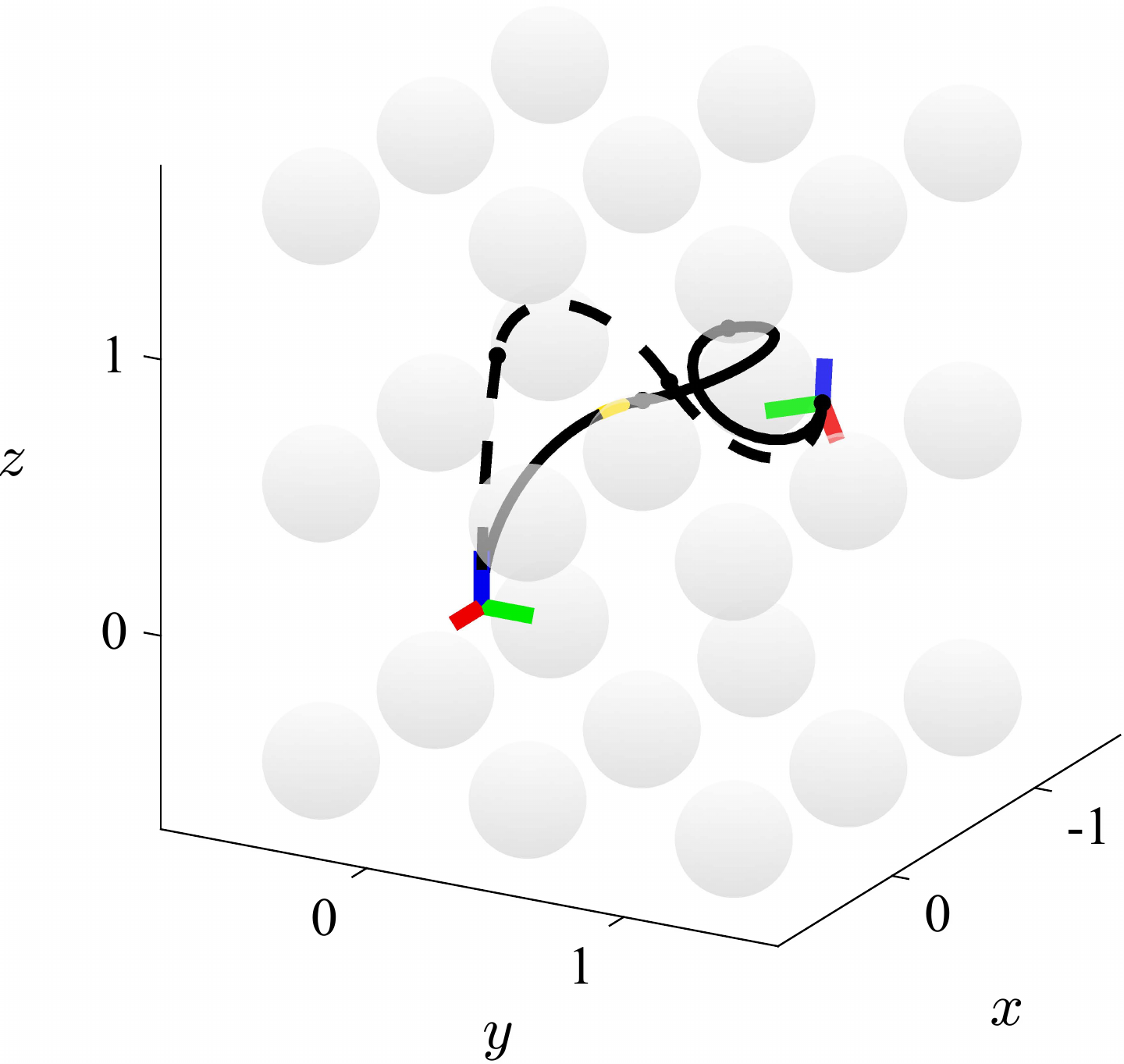}\label{fig_4_solutions_1}}
  \subfigure[]{
  \includegraphics[width = 0.42 \linewidth]{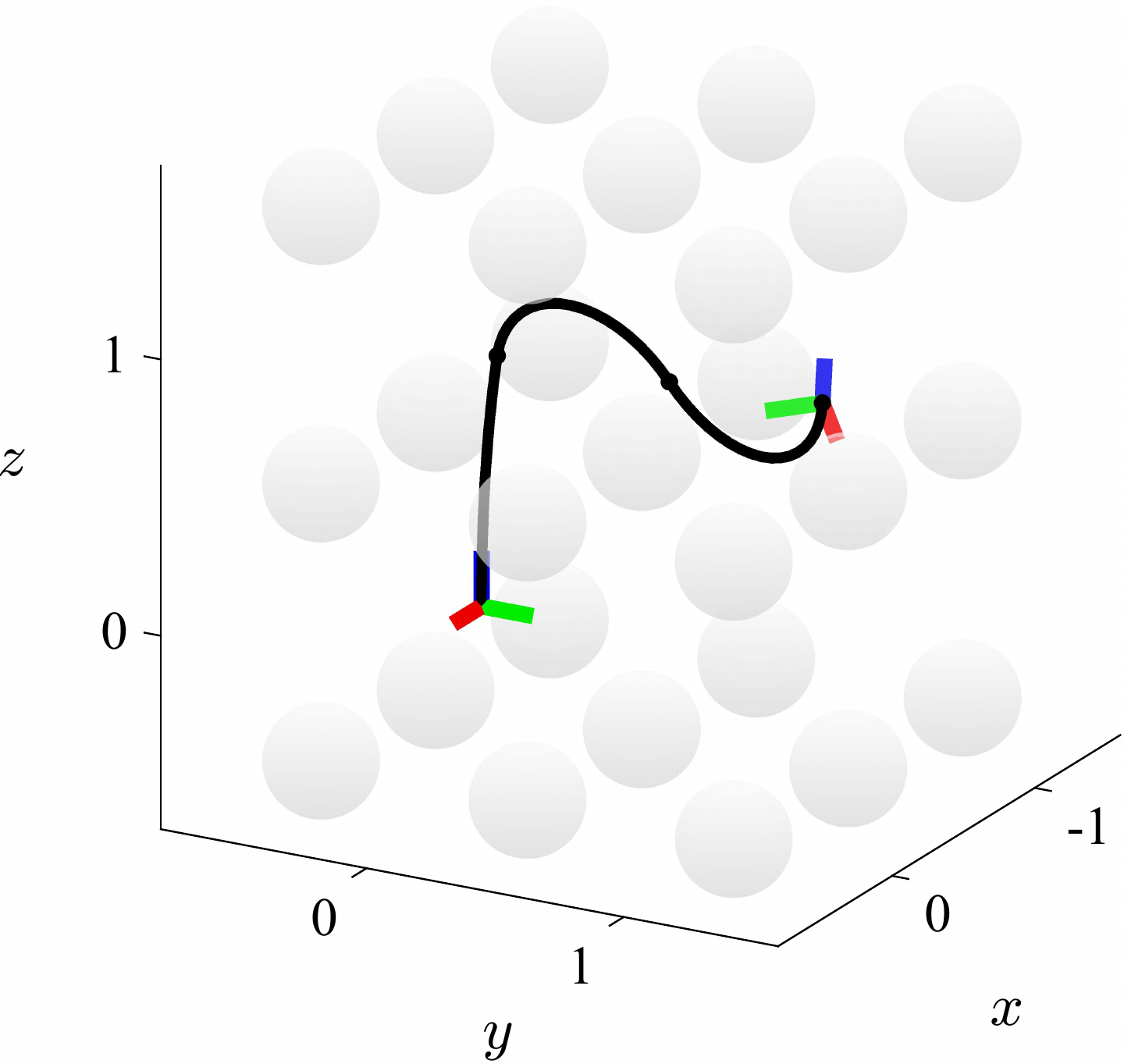}\label{fig_4_solutions_2}}
  \subfigure[]{
  \includegraphics[width = 0.42 \linewidth]{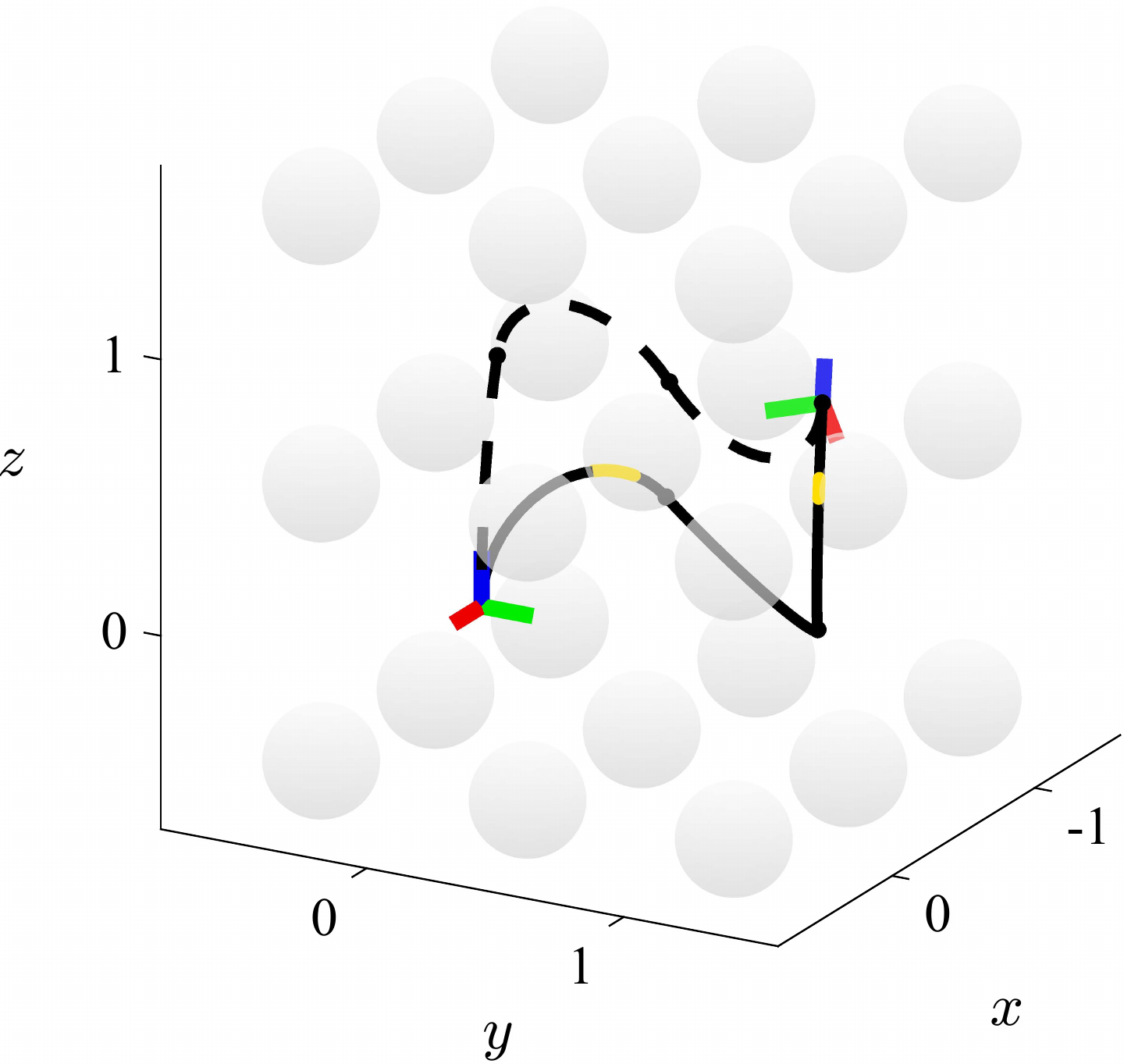}\label{fig_4_solutions_3}}
  \subfigure[]{
  \includegraphics[width = 0.42 \linewidth]{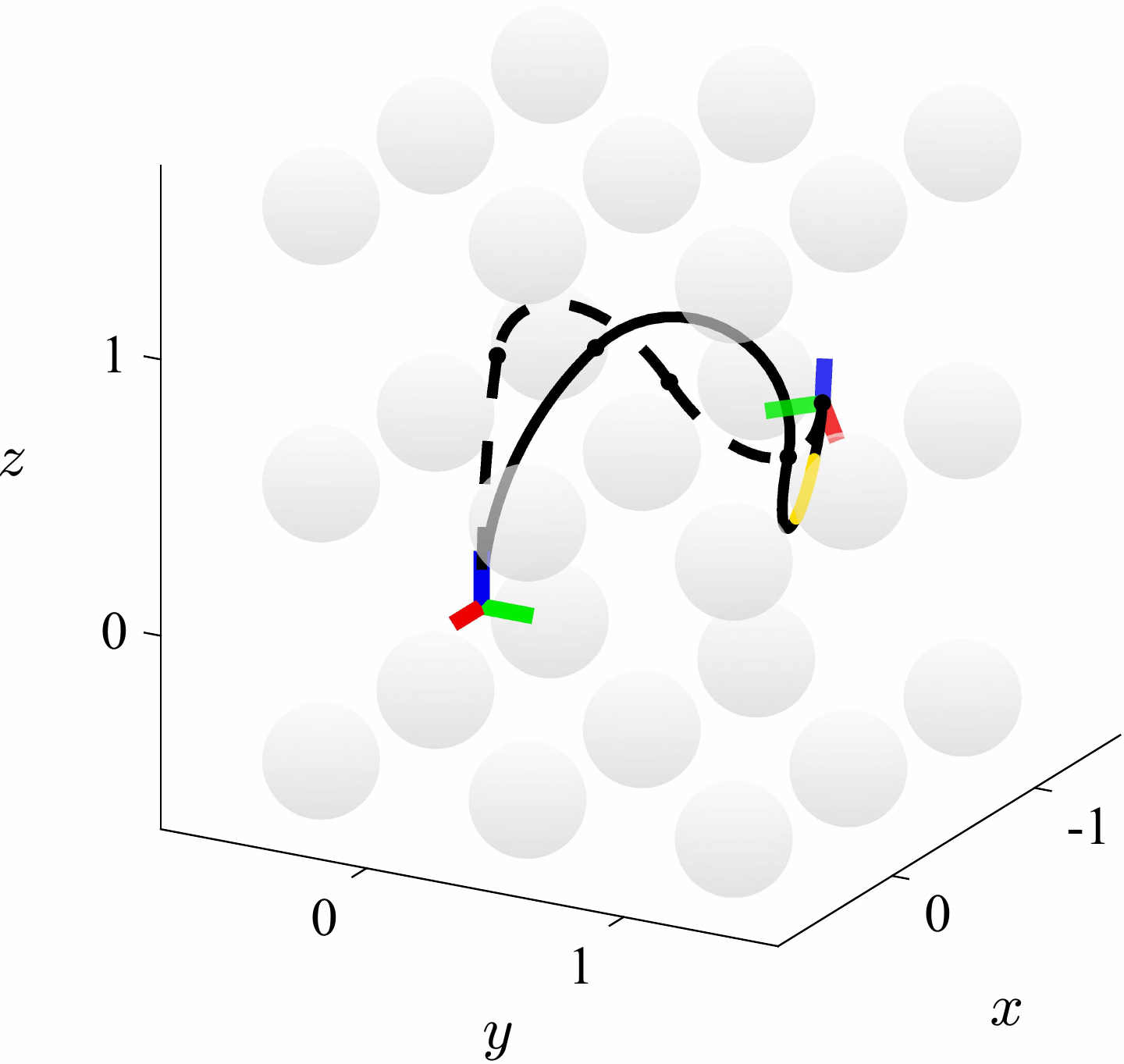}\label{fig_4_solutions_4}}
  \caption{We choose an end rotation $q = \cos\pbrac{\alpha / 2} + \sin\pbrac{\alpha / 2} \bsym{\omega}$ and translation $\bsym{r} = \pmat{-0.4 & 1.1 & 0.8}^\trans$, where $\alpha = 15 \pi / 16, \bsym{\omega} = 0.48 \bsym{i} + 0.1\sqrt{3} \bsym{j} -0.86 \bsym{k}$. With our solver, we obtain (a) the error curve with a total of 4 equilibria, which are marked in boxes; (b-e) multiple solutions corresponding to equilibria from left to right. Collisions are painted yellow.}
  \label{fig_4_solutions}
\end{figure}

\begin{figure}[t]
  \centering
  \subfigure[]{
  \includegraphics[width = 0.4 \linewidth]{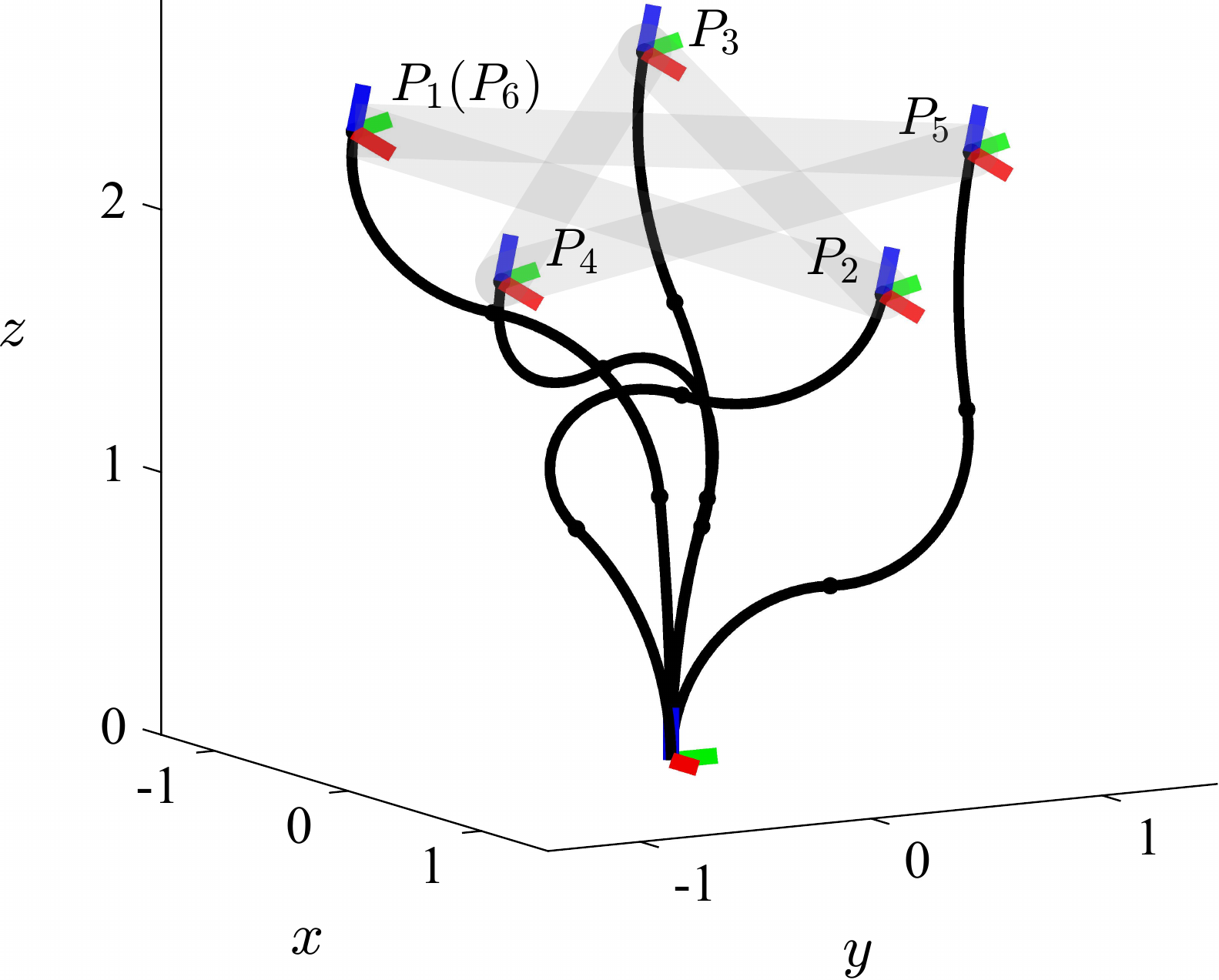}
  \label{fig_path_star_grad}}
  \subfigure[]{
  \includegraphics[width = 0.4 \linewidth]{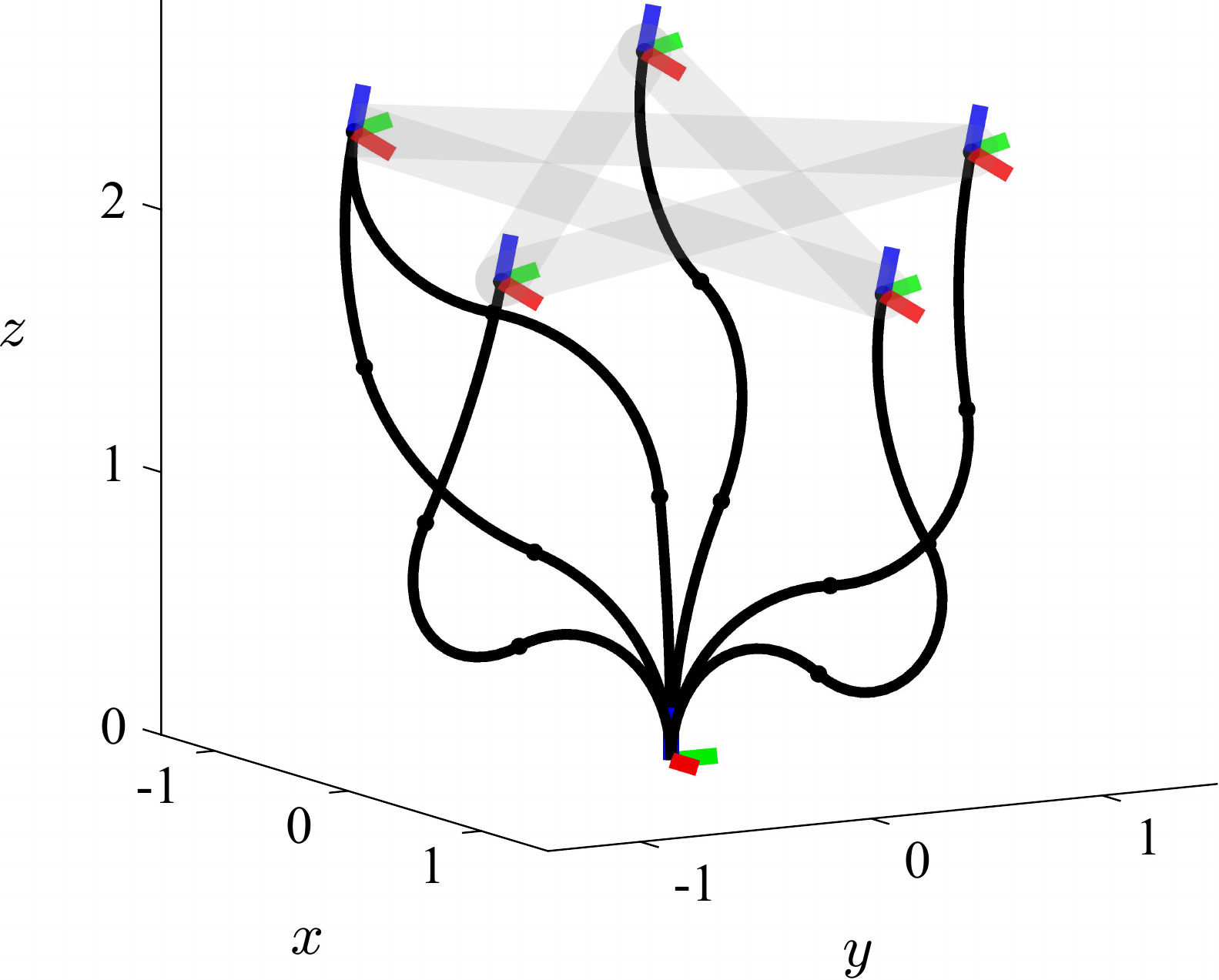}
  \label{fig_path_star_newton}}
  \subfigure[]{
  \includegraphics[width = 0.4 \linewidth]{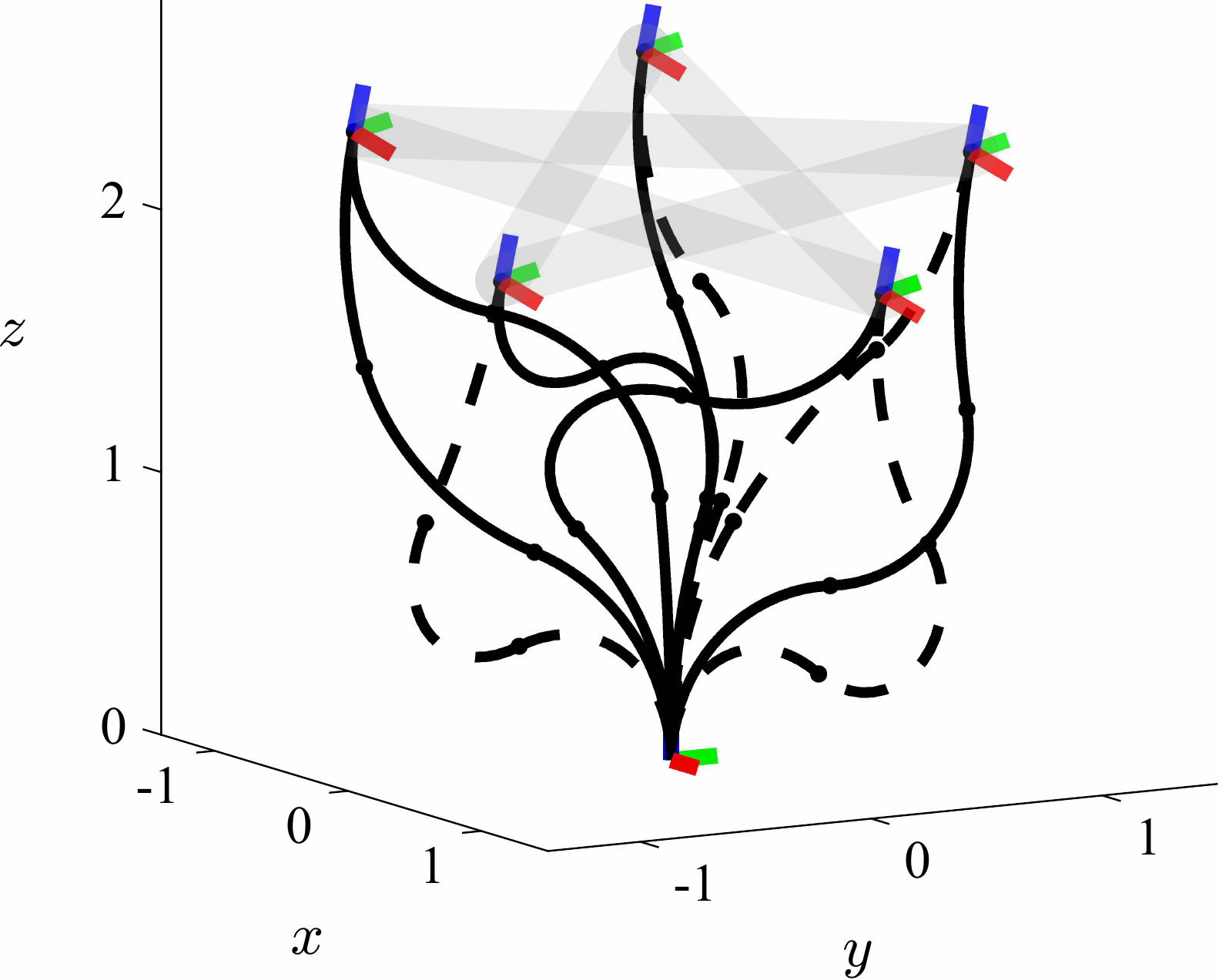}
  \label{fig_path_star_ours}}
  \subfigure[]{
  \includegraphics[width = 0.4 \linewidth]{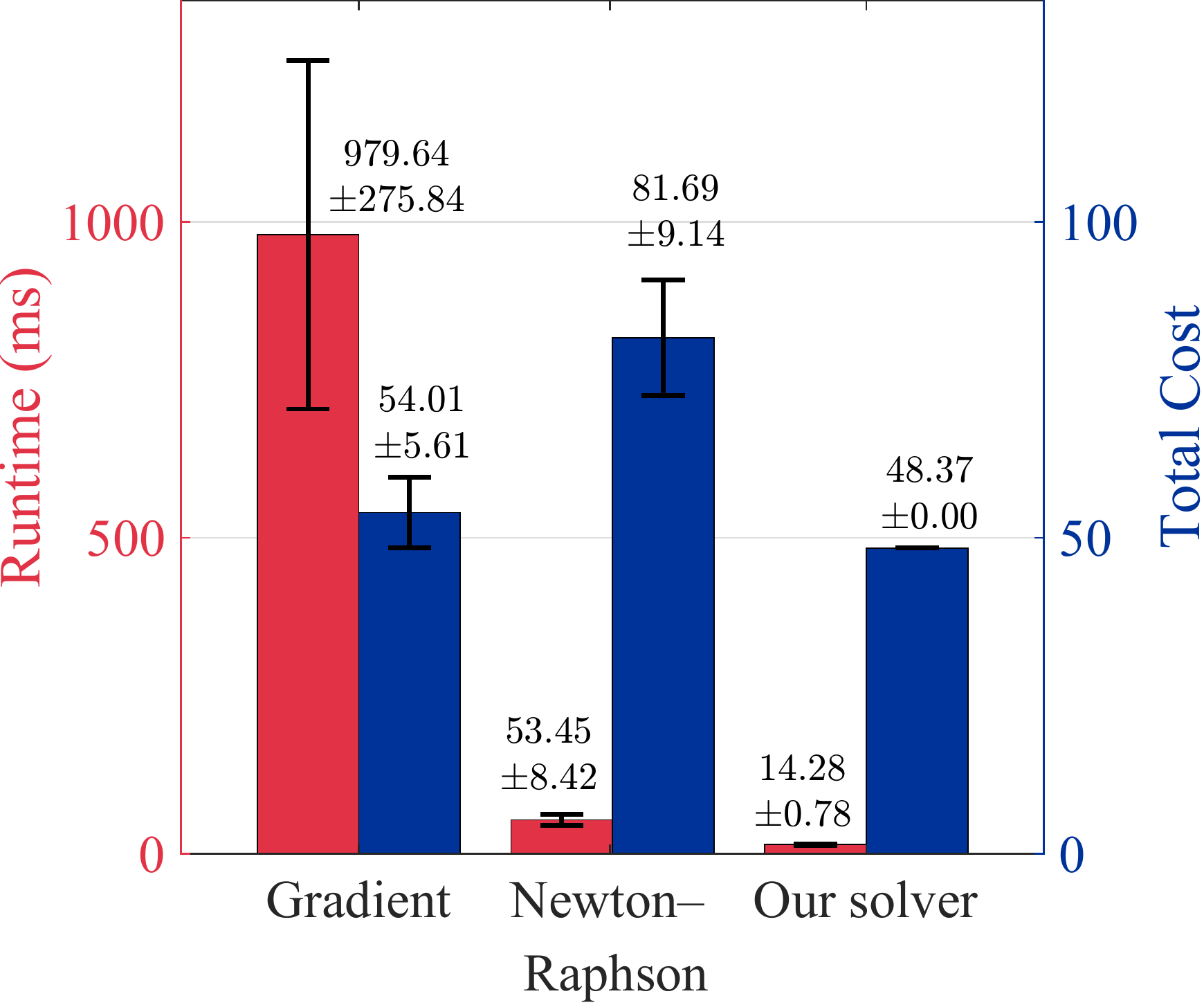}
  \label{fig_path_star_bar}}
  \caption{Planning Results of a pentagram-shaped path $\pbrac{P_1, P_2, \dots, P_6}$ using (a) the gradient method, (b) the Newton-Raphson method, (c) our solver. The dashed lines represent the multiple solutions. Using different random seeds, the runtime and total cost are summarised in (d).}
  \label{fig_path_star}
\end{figure}

\subsection{Path Planning}\label{path_planning}

In this scenario, instead of random sampling, we specify a path composed of a sequence of via points by giving the position and orientation of each point.

\textbf{Pentagram-Shaped Path.} Figure~\ref{fig_path_star} illustrates the planning result with 3 different methods. The robot passes through points $P_1, P_2, \dots, P_6$ in order. For the gradient method and Newton-Raphson method, the solution of the current pose would be used as the initial value for the next pose. For our solver, with multiple solutions obtained at each point, the one with minimal cost is chosen greedily. The cost is calculated by summing up the squared differences of model parameters.

\textbf{Comparison on Cost and Efficiency.} We repeat the experiment with 10 random seeds, and the error bars represent the 95\% Jeffreys confidence intervals \citep{cai2005onesided}. The total cost of planning varies from different seeds, revealing the degree of dependency on initial values. Hence the 0 deviation of cost reports that our algorithm is deterministic. The solver reduces about 10\% of the total cost relative to the gradient method. In terms of efficiency, the runtime is reduced by about 73\% compared with the Newton-Raphson method. Numerical methods do perform better when the previous solution is used as the initial value rather than a random one. Nevertheless, the smaller deviation of runtime reflects that our solver is more stable.

\textbf{Optimal RSS-Shaped Path.} We construct a sequence with 100 via points that form the shape of RSS. The orientations are all identical, whose $x$-axis points vertically downwards and is then rotated by $-\pi/4$ radians. A spherical obstacle with a diameter of 0.8 is placed by the side of the path. While the solver calculates the solutions at each via point, we discard those that collide with the obstacle. The cost is the same as defined in the previous path planning experiment. By applying dynamic programming, we achieve the \textit{optimal} path as shown in Figure~\ref{fig_path_rss}. For clarity, we plot the solutions every 3 points.

\section{Conclusion}\label{conclusion}

In this paper, we present an efficient multi-solution solver to deal with the inverse kinematics problem of 3-section constant-curvature robots. Our theory simplifies the problem to an equivalent one-dimensional problem and gives an error bound for the following approximation. Furthermore, a circular traversal is applied to the approximated problem, which brings the resolution completeness, then through a few steps of iterations, we arrive at the solution of the inverse kinematics problem. The experimental results demonstrate a better efficiency and higher success rate of our solver than benchmark approaches when one solution is required. In addition, the ability of multi-solution finding allows optimal path planning in a space with obstacles.

\section*{Acknowledgments}

This work was supported by the National Key R\&D Program of China (2021ZD0114500), the  National Natural Science Foundation of China (T2293724, 62173293), Key R\&D Program of Zhejiang (2022C01022), Zhejiang Provincial Natural Science Foundation of China (LD22E050007), Fundamental Research Funds for the Zhejiang Provincial Universities (2021XZZX021).


\bibliographystyle{plainnat}
\bibliography{refs.bib}

\clearpage
\appendix

\subsection{A Catalog of Common Formulae}\label{appendix_catalog_of_common_formulae}

A unit quaternion $q$ is written as a sum of real and imaginary parts in~(\ref{qabcd}), but this kind of notation is not always convenient in our proof. Equivalently, we can also represent the same quaternion $q$ in the form of a column vector, namely,
\begin{equation}
  q = \pmat{\delta \\ \bsym{\varepsilon}}, \quad \delta \in \mathbb{R}, \quad \bsym{\varepsilon} \in \mathbb{R}^3, \quad \sqrt{\delta^2 + \norm{\bsym{\varepsilon}}^2} = 1.
\end{equation}
Then the conjugate of $q$ becomes
\begin{equation}
  q^* = \pmat{\delta \\ -\bsym{\varepsilon}}.
\end{equation}
By definition, we compute the multiplication of two quaternions $q_1$ and $q_2$ as
\begin{equation}
  q_1 \otimes q_2 = \pmat{\delta_1 \delta_2 - \bsym{\varepsilon}_1^\trans \bsym{\varepsilon}_2 \\
  \delta_1 \bsym{\varepsilon}_2 + \delta_2 \bsym{\varepsilon}_1 + \bsym{\varepsilon}_1^\wedge \bsym{\varepsilon}_2},
\end{equation}
where the wedge operator $\pbrac{\cdot}^\wedge$ produces the usual skew-symmetric matrix,
\begin{equation}
  \bsym{\varepsilon}^\wedge = \pmat{0 & -\varepsilon_3 & \varepsilon_2 \\ \varepsilon_3 & 0 & -\varepsilon_1 \\ -\varepsilon_2 & \varepsilon_1 & 0}
\end{equation}
Since the quaternion multiplication is non-commutative, we introduce the left and right compound operators \citep{sola2017quaternion,barfoot2011pose} by defining $\pbrac{\cdot}^+: \mathbb{H} \to \mathbb{R}^{4 \times 4}$ and $\pbrac{\cdot}^\oplus: \mathbb{H} \to \mathbb{R}^{4 \times 4}$,
\begin{equation}
  \label{upplus}
  q_1^+ = \pmat{\delta_1 & -\bsym{\varepsilon}_1^\trans \\
  \bsym{\varepsilon}_1 & \delta_1 \bsym{I} + \bsym{\varepsilon}_1^\wedge},\quad
  q_2^\oplus = \pmat{\delta_2 & -\bsym{\varepsilon}_2^\trans \\
  \bsym{\varepsilon}_2 & \delta_2 \bsym{I} - \bsym{\varepsilon}_2^\wedge},
\end{equation}
so that the quaternion multiplication can be expressed in matrix notations, i.e.,
\begin{equation}
  q_1 \otimes q_2 = q_1^+ q_2 = q_2^\oplus q_1.
\end{equation}
It can be extended to a vector $\bsym{r} \in \mathbb{R}^3$ if it is written in the form of a pure quaternion $\bsym{r} = r_x \bsym{i} + r_y \bsym{j} + r_z \bsym{k}$, then the following product represents the rotation of $\bsym{r}$, namely,
\begin{equation}
  q \otimes \bsym{r} \otimes q^* = q^+ \pbrac{q^*}^\oplus \bsym{r} = \pbrac{\bsym{\varepsilon} \bsym{\varepsilon}^\trans + \pbrac{\delta \bsym{I} + \bsym{\varepsilon}^\wedge}^2} \bsym{r}.
\end{equation}

\subsection{Proof of Lemma~\ref{lemma1}}\label{appendix_proof_of_lemma1}

\begin{proof}
  We first reorder~(\ref{s2rot}) as $q_2 = q_1^* \otimes q$ since $q_1^{-1} = q_1^*$ for unit quaternions,
  then apply the matrix notation to get
  \begin{equation}
    q_2 = q^\oplus q_1^*,
  \end{equation}
  or
  \begin{equation}
    \label{lemma1pfeq1}
    \pmat{a_2 \\ b_2 \\ c_2 \\ d_2}
    =
    \pmat{a & -b & -c & -d \\ b &  a &  d & -c \\ c & -d &  a &  b \\ d &  c & -b & a}
    \pmat{a_1 \\ -b_1 \\ -c_1 \\ -d_1}.
  \end{equation}
  It has already been determined that $d_\lambda = 0$.
  So we separate the last row of~(\ref{lemma1pfeq1}) from another three to get two independent constraints.
  By rearranging the elements in the order of $\pmat{c_\lambda & -b_\lambda & a_\lambda}^\trans = \hbsym{r}_\lambda$,
  we obtain an identity of $\hbsym{r}_1$ from the last row,
  $$
  \pmat{b & c & d} \hbsym{r}_1 = 0,
  $$
  as well as a relationship of $\hbsym{r}_1$ and $\hbsym{r}_2$ from another three rows,
  $$
  \bsym{A} \hbsym{r}_1 = \hbsym{r}_2.
  $$
  This shows the necessity.

  Geometrically,~(\ref{lemma1eq1}) refers to a plane through the origin in the space $\mathbb{R}^3$
  with a normal vector $\bsym{n} = \pmat{b & c & d}^\trans$.
  This plane intersects the unit sphere in the space since $\hbsym{r}_1 \in \mathbb{S}^2 \subseteq \mathbb{R}^3$.
  The intersection is simply a circle
  \begin{equation}
    \label{lemma1pfeq2}
    C \pbrac{\bsym{n}} = \{ \bsym{v} \in \mathbb{S}^2: \bsym{n} \cdot \bsym{v} = 0 \}.
  \end{equation}

  Next we show the sufficiency to complete this proof.
  The proof of necessity is invertible, so we only need to verify that the solution is feasible by proving
  $\hbsym{r}_2 \in \bsym{A} C\pbrac{\bsym{n}} \subseteq \mathbb{S}^2$.
  Let $\hbsym{r}_1 = \pmat{c_1 & -b_1 & a_1}^\trans \in C\pbrac{\bsym{n}}$, then
  \begin{equation}
    \begin{aligned}
      \norm{\bsym{A} \hbsym{r}_1}^2 = &\ \hbsym{r}_1^\trans \bsym{A}^\trans \bsym{A} \hbsym{r}_1 \\
      = &\ \pbrac{1 - b^2} c_1^2 + 2 b c b_1 c_1 \\
      &\ + \pbrac{1 - c^2} b_1^2 + 2 c d a_1 b_1 \\
      &\ + \pbrac{1 - d^2} a_1^2 - 2 b d a_1 c_1 \\
      = &\ c_1^2 + b_1^2 + a_1^2 - \pbrac{b c_1 - c b_1 + d a_1}^2.
    \end{aligned}
  \end{equation}
  Since $\hbsym{r}_1 \in C\pbrac{\bsym{n}}$, we have $\norm{\hbsym{r}_1} = \sqrt{c_1^2 + b_1^2 + a_1^2} = 1$
  and $\bsym{n} \cdot \hbsym{r}_1 = b c_1 - c b_1 + d a_1 = 0$, therefore
  $$
  \norm{\hbsym{r}_2} = \norm{\bsym{A} \hbsym{r}_1} = 1.
  $$
  This shows the sufficiency.
\end{proof}

\subsection{Corollary}\label{appendix_corollary}

We observe a consequence after proving Lemma~\ref{lemma1}.

\begin{corollary}\label{cor1}
  Given an end rotation $q = \pmat{a & b & c & d}^\trans$, for those 2-section continuum robots with identical end rotation, their endpoints fall on the plane passing through the origin and with a normal vector $\bsym{n} = \pmat{b & c & d}^\trans$, i.e.,
  \begin{equation}
    \label{cor1eq1}
    \bsym{n} \cdot \bsym{r} = 0.
  \end{equation}
\end{corollary}

\begin{proof}
  Rewrite~(\ref{s2trans}) as
  \begin{equation}
    \label{cor1pfeq1}
    \begin{aligned}
      \bsym{r} &= \rho\pbrac{a_1, L_1} \hbsym{r}_1 + \rho\pbrac{a_2, L_2} q_1 \otimes \hbsym{r}_2 \otimes q_1^* \\
      &= \rho\pbrac{a_1, L_1} \hbsym{r}_1 + \rho\pbrac{a_2, L_2} q_1^+ \pbrac{q_1^*}^\oplus \hbsym{r}_2.
    \end{aligned}
  \end{equation}
  We have known that the 2-section continuum robots with identical end rotation
  must all meet the perpendicular constraint $\bsym{n} \cdot \hbsym{r}_1 = 0$ referred to in~(\ref{lemma1eq1}).
  So the first term on the right side of~(\ref{cor1pfeq1}) is perpendicular to $\bsym{n}$.
  So does the second term because if we consider the inner product $\bsym{n} \cdot q_1^+ \pbrac{q_1^*}^\oplus \hbsym{r}_2$,
  then plug~(\ref{lemma1eq2}) into this expression and it becomes 
  \begin{equation}
    \begin{aligned}
      &\ \bsym{n} \cdot q_1^+ \pbrac{q_1^*}^\oplus \bsym{A} \hbsym{r}_1 \\
      = &\ a d a_1 - a b c_1 + a c b_1 + 2 b d b_1 + 2 c d c_1 \\
      &\ + 2 \pbrac{a b c_1 - a c b_1 - b d b_1 - c d c_1} \pbrac{a_1^2 + b_1^2 + c_1^2} \\
      = &\ a \pbrac{d a_1 + b c_1 - c b_1}.
    \end{aligned}
  \end{equation}
  As we have encountered in the proof of Lemma~\ref{lemma1},
  we get $\bsym{n} \cdot q_1^+ \pbrac{q_1^*}^\oplus \hbsym{r}_2 = a \cdot \bsym{n}^\trans \hbsym{r}_1 = 0$.
  This proves $\bsym{n} \cdot \bsym{r} = 0$. Hence, the endpoint falls on that given plane.
\end{proof}

\subsection{Proof of Lemma~\ref{lemma2}}\label{appendix_proof_of_lemma2}

\begin{proof}
  Let $\bsym{w}_2 = \pmat{-c_2 & b_2 & a_2}^\trans$,
  then by using matrix notations~(\ref{upplus}) and~(\ref{upstar}), we find that
  \begin{equation}
    \label{lemma2pfeq1}
    \begin{aligned}
      &\quad q_1^+ \pbrac{q_1^*}^\oplus \hbsym{r}_2 \\
      &= \pmat{2c_1^2-1 & -2b_1c_1 & 2a_1c_1 \\
      -2b_1c_1 & 2b_1^2-1 & -2a_1b_1 \\
      2a_1c_1 & -2a_1b_1 & 2a_1^2-1}
      \pmat{-c_2 \\ b_2 \\ a_2} \\
      &= (2 \hbsym{r}_1 \hbsym{r}_1^\trans - \bsym{I}) \bsym{w}_2.
    \end{aligned}
  \end{equation}
  For the last equality above, we use the fact that $\hbsym{r}_1^\trans \hbsym{r}_1 = \norm{\bsym{r}_1}^2 = 1$.
  Combine~(\ref{cor1pfeq1}) and~(\ref{lemma2pfeq1}) to get
  \begin{equation}
    \label{lemma2pfeq2}
    \bsym{r} - \rho\pbrac{a_1, L_1}\hbsym{r}_1
    = \rho\pbrac{a_2, L_2} \pbrac{2 \hbsym{r}_1 \hbsym{r}_1^\trans - \bsym{I}} \bsym{w}_2.
  \end{equation}
  We also notice the matrix $2 \hbsym{r}_1 \hbsym{r}_1^\trans - \bsym{I}$ is symmetric and involutory, namely,
  $$
  \pbrac{2 \hbsym{r}_1 \hbsym{r}_1^\trans - \bsym{I}}^\trans = \pbrac{2 \hbsym{r}_1 \hbsym{r}_1^\trans - \bsym{I}}, \pbrac{2 \hbsym{r}_1 \hbsym{r}_1^\trans - \bsym{I}}^2 = \bsym{I}.
  $$
  So if we multiply~(\ref{lemma2pfeq2}) by $2 \hbsym{r}_1 \hbsym{r}_1^\trans - \bsym{I}$
  and then denote $\bsym{v} = \bsym{r} - \rho\pbrac{a_1, L_1}\hbsym{r}_1$, we can get
  \begin{equation}
    \label{lemma2pfeq3}
    \pbrac{2 \hbsym{r}_1 \hbsym{r}_1^\trans - \bsym{I}} \bsym{v} = \rho\pbrac{a_2, L_2} \bsym{w}_2.
  \end{equation}
  This equation is equivalent to~(\ref{s2trans}) because the derivation above can be inverted.
  Note that the left side only involves variables of the first arc section and the right side only of the second.
  To make it solvable, a necessary condition is the norm of both sides are equal.
  According to the definition of $\bsym{w}_2$,
  the norm of $\bsym{w}_2$ equals to the norm of $\bsym{r}_2$, which is exactly $1$, so we must have
  \begin{equation}
    \label{lemma2pfeq4}
    \norm{\bsym{v}} = \rho\pbrac{a_2, L_2}.
  \end{equation}
  Now $a_2$ can be calculated from the last row (the third dimension) of~(\ref{lemma2pfeq3}) by
  \begin{equation}
    \label{lemma2pfeq5}
    a_2 = \frac{\bsym{u}\cdot\bsym{v}}{\rho\pbrac{a_2, L_2}} = \frac{\bsym{u}\cdot\bsym{v}}{\norm{\bsym{v}}},
  \end{equation}
  where $\bsym{u} = \pmat{2 a_1 c_1 & -2 a_1 b_1 & 2 a_1^2 - 1}^\trans$ is the transpose of the third row of the matrix $2 \hbsym{r}_1 \hbsym{r}_1^\trans - \bsym{I}$, as defined in~(\ref{lemma3eq2}). So the inner product $\bsym{u}\cdot\bsym{v}$ extracts the last component of the vector of the left part of~(\ref{lemma2pfeq3}).
  
  Substitute~(\ref{lemma2pfeq5}) into~(\ref{lemma2pfeq4}) yields~(\ref{lemma2eq1}).
  Substitute~(\ref{lemma2pfeq4}) into~(\ref{lemma2pfeq3}) and then we can calculate $\bsym{w}_2$ by
  \begin{equation}
    \bsym{w}_2 = \pbrac{2 \bsym{r}_1 \bsym{r}_1^\trans - \bsym{I}} \frac{\bsym{v}}{\norm{\bsym{v}}},
  \end{equation}
  so $\hbsym{r}_2$ can be recovered from $\bsym{w}_2$ using elementary operations.
  This shows~(\ref{lemma2eq4}).
\end{proof}

\subsection{Proof of Lemma~\ref{lemma3}}\label{appendix_proof_of_lemma3}

We list some identities before we turn to the proof.

\begin{proposition}\label{prop2}
  Let $\bsym{B}$ be defined as~(\ref{lemma3eq2}),
  and denote $\bsym{m} = \pmat{c & -b & a}^\trans$, then for $\lambda = 1, 2, 3$, we have
  \begin{enumerate}[\rm\bfseries(a)]
    \item $\quad \hbsym{r}_\lambda^\trans \bsym{B} \hbsym{r}_\lambda = \hbsym{r}_\lambda^\trans \bsym{B}^\trans \hbsym{r}_\lambda = d$.
    \item $\quad \hbsym{r}_\lambda^\trans \pbrac{\bsym{B}^\trans \bsym{B} + \bsym{m} \bsym{m}^\trans} \hbsym{r}_\lambda = 1$.
  \end{enumerate}
\end{proposition}

\begin{proof}
  The proof of the identities is straightforward using the fact that $q$ is a unit quaternion and $\hbsym{r}_\lambda$ is a unit vector.
  \textbf{(a)} Expand $\hbsym{r}_\lambda^\trans \bsym{B} \hbsym{r}_\lambda$ and we have $\hbsym{r}_\lambda^\trans \bsym{B} \hbsym{r}_\lambda = d \pbrac{a_\lambda^2 + b_\lambda^2 + c_\lambda^2} = d$. Since $\bsym{B} - \bsym{B}^\trans$ is antisymmetric, we have $\hbsym{r}_\lambda^\trans \pbrac{\bsym{B} - \bsym{B}^\trans} \hbsym{r}_\lambda = 0$, so $\hbsym{r}_\lambda^\trans \bsym{B} \hbsym{r}_\lambda = \hbsym{r}_\lambda^\trans \bsym{B}^\trans \hbsym{r}_\lambda = d$.
  \textbf{(b)} Expand the left and we have $\hbsym{r}_\lambda^\trans \pbrac{\bsym{B}^\trans \bsym{B} + \bsym{m} \bsym{m}^\trans} \hbsym{r}_\lambda = \pbrac{a^2 + b^2 + c^2 + d^2} \pbrac{a_\lambda^2 + b_\lambda^2 + c_\lambda^2} = 1 \cdot 1 = 1$.
\end{proof}

Knowing this, the proof of Lemma~\ref{lemma3} can be completed.

\begin{proof}[Proof of Lemma~\ref{lemma3}]
  As mentioned, we partition the three sections into two groups.
  There are only two ways to make the partition,
  namely, $\{1, 2\}\cup\{3\}$ and $\{1\}\cup\{2, 3\}$.
  We will see soon that the former leads to the case $\lambda = 3$,
  and the latter to the case $\lambda = 1$.
  First we show the case $\lambda = 3$. Let
  \begin{equation}
    q_e = q_1 \otimes q_2 = q \otimes q_3^{-1} = q \otimes q_3^*,
  \end{equation}
  then through basic matrix multiplications we find
  \begin{equation}
    \label{lemma3pfeq1}
    q_e = q^+ q_3^* = \pmat{\delta \delta_3 + \bsym{\varepsilon}^\trans \bsym{\varepsilon}_3 \\ \delta_3 \bsym{\varepsilon} - \delta \bsym{\varepsilon}_3 - \bsym{\varepsilon}^\wedge \bsym{\varepsilon}_3} = \pmat{\bsym{m}^\trans \hbsym{r}_3 \\ \bsym{B} \hbsym{r}_3}.
  \end{equation}
  Here we keep adopting the symbols $\bsym{m}$ and $\bsym{B}$ defined in Proposition~\ref{prop2}.
  According to Corollary~\ref{cor1}, we know that the endpoint of the first two
  arc sections lies on the plane passing the origin with a normal vector $\bsym{n}_e = \pmat{b_e & c_e & d_e}^\trans = \bsym{B} \hbsym{r}_3$.
  Denote
  \begin{equation}
    \label{lemma3pfeq21}
    \bsym{r}_e = \bsym{r} - q_e \otimes \bsym{r}_3 \otimes q_e^*,
  \end{equation}
  then Lemma~\ref{lemma1} tells us that
  \begin{equation}
    \label{lemma3pfeq22}
    \bsym{n}_e \cdot \bsym{r}_e = 0.
  \end{equation}
  Replace $\bsym{r}_e$ in~(\ref{lemma3pfeq22}) with~(\ref{lemma3pfeq21}) we get
  \begin{equation}
    \bsym{r}^\trans \bsym{B} \hbsym{r}_3 - \rho\pbrac{a_3, L_3}
    \bsym{n}_e \cdot q_e \otimes \bsym{r}_3 \otimes q_e^* = 0,
  \end{equation}
  which is close to the result we want. It suffices to show that
  \begin{equation}
    \bsym{n}_e \cdot q_e \otimes \bsym{r}_3 \otimes q_e^* = d.
  \end{equation}
  Note that
  \begin{equation}
    \begin{aligned}
      &\quad q_e \otimes \bsym{r}_3 \otimes q_e^* \\
      &= \pbrac{\bsym{B} \hbsym{r}_3} \pbrac{\bsym{B} \hbsym{r}_3}^\trans \hbsym{r}_3 + \pbrac{\bsym{m}^\trans \hbsym{r}_3}^2 \hbsym{r}_3 \\
      &\quad + 2 \pbrac{\bsym{m}^\trans \hbsym{r}_3} \pbrac{\bsym{B} \hbsym{r}_3}^\wedge \hbsym{r}_3 + \pbrac{\bsym{B} \hbsym{r}_3}^\wedge \pbrac{\pbrac{\bsym{B} \hbsym{r}_3}^\wedge \hbsym{r}_3},
    \end{aligned}
  \end{equation}
  and therefore we have
  \begin{equation}
    \begin{aligned}
      &\quad\pbrac{\bsym{B} \hbsym{r}_3} \cdot q_e \otimes \bsym{r}_3 \otimes q_e^* \\
      &= \pbrac{\bsym{B} \hbsym{r}_3}^\trans \pbrac{\bsym{B} \hbsym{r}_3} \pbrac{\hbsym{r}_3^\trans \bsym{B}^\trans \hbsym{r}_3}
      + \pbrac{\bsym{m}^\trans \hbsym{r}_3}^2 \pbrac{\hbsym{r}_3^\trans \bsym{B}^\trans \hbsym{r}_3} + 0 + 0 \\
      &= \pbrac{\hbsym{r}_3^\trans \bsym{B}^\trans \bsym{B} \hbsym{r}_3
      + \hbsym{r}_3^\trans \bsym{m} \bsym{m}^\trans \hbsym{r}_3}
      \pbrac{\hbsym{r}_3^\trans \bsym{B}^\trans \hbsym{r}_3} \\
      &= d
    \end{aligned}
  \end{equation}
  The last equality follows from Proposition~\ref{prop2}. So~(\ref{lemma3eq1}) holds for $\lambda = 3$.

  The proof of case $\lambda = 1$ is simpler. We use the partition $\{1\}\cup\{2, 3\}$. Let
  \begin{equation}
    q_e' = q_2 \otimes q_3 = q_1^{-1} \otimes q = q_1^* \otimes q,
  \end{equation}
  and we can find a matrix $\bsym{B}'$ such that
  \begin{equation}
    \label{lemma3pfeq23}
    q_e' = \pbrac{q_1^*}^+ q = \pmat{\delta_1 \delta + \bsym{\varepsilon}_1^\trans \bsym{\varepsilon} \\ \delta_1 \bsym{\varepsilon} - \delta \bsym{\varepsilon}_1 - \bsym{\varepsilon}_1^\wedge \bsym{\varepsilon}} = \pmat{\cdot \\ \bsym{B}' \hbsym{r}_1}.
  \end{equation}
  Rearrange~(\ref{s3trans}) and we have
  \begin{equation}
    \label{lemma3pfeq3}
    q_1^* \otimes \pbrac{\bsym{r} - \bsym{r}_1} \otimes q_1
    = \bsym{r}_2 + q_2 \otimes \bsym{r}_3 \otimes q_2^*.
  \end{equation}
  Again, Corollary~\ref{cor1} specifies a plane through the origin with a normal vector $\bsym{n}'_e = \bsym{B}' \hbsym{r}_1$, and says the vector in $\mathbb{R}^3$ to the left of~(\ref{lemma3pfeq3}) lies on that plane. If $\bsym{R}_1^* \in \mathsf{SO}(3)$ is the rotation matrix of $q_1^*$, then we have
  \begin{equation}
    \label{lemma3pfeq4}
    \pbrac{\bsym{B}' \hbsym{r}_1} \cdot \bsym{R}_1^* \pbrac{\bsym{r} - \bsym{r}_1} = 0.
  \end{equation}
  Since $\pbrac{q_1^*}^{-1} = q_1$, we know that the corresponding rotatioin matrix of $q_1$ is $\bsym{R}_1 = \pbrac{\bsym{R}_1^*}^{-1} = \pbrac{\bsym{R}_1^*}^\trans$.
  So we have
  \begin{equation}
    \label{lemma3pfeq51}
    \pbrac{\bsym{R}_1^*}^\trans \bsym{B}' \hbsym{r}_1 = \bsym{R}_1 \bsym{B}' \hbsym{r}_1 = q_1 \otimes \pbrac{\bsym{B}' \hbsym{r}_1} \otimes q_1^*.
  \end{equation}
  Using~(\ref{lemma3pfeq23}) we find that
  \begin{equation}
    \label{lemma3pfeq52}
    \begin{aligned}
      &\quad q_1 \otimes \pbrac{\bsym{B}' \hbsym{r}_1} \otimes q_1^* \\
      &= \pbrac{\bsym{\varepsilon}_1 \bsym{\varepsilon}_1^\trans + \pbrac{\delta_1 \bsym{I} + \bsym{\varepsilon}_1^\wedge}^2} \pbrac{\delta_1 \bsym{\varepsilon} - \delta \bsym{\varepsilon}_1 - \bsym{\varepsilon}_1^\wedge \bsym{\varepsilon}} \\
      &= \delta_1 \bsym{\varepsilon} - \delta \bsym{\varepsilon}_1 - \bsym{\varepsilon}^\wedge \bsym{\varepsilon}_1 \\
      &= \bsym{B} \hbsym{r}_1.
    \end{aligned}
  \end{equation}
  So~(\ref{lemma3pfeq51}) and~(\ref{lemma3pfeq52}) leads to
  \begin{equation}
    \label{lemma3pfeq53}
    \pbrac{\bsym{R}_1^*}^\trans \bsym{B}' \hbsym{r}_1 = \bsym{B} \hbsym{r}_1.
  \end{equation}
  We combine~(\ref{lemma3pfeq4}) and~(\ref{lemma3pfeq53}) and apply Proposition~\ref{prop2}(a) to deduce that
  \begin{equation}
    \begin{aligned}
      0 &= \pbrac{\bsym{B}' \hbsym{r}_1} \cdot \bsym{R}_1^* \pbrac{\bsym{r} - \bsym{r}_1} \\
      &= \bsym{r}^\trans \pbrac{\bsym{R}_1^*}^\trans \bsym{B}' \hbsym{r}_1
      - \bsym{r}_1^\trans \pbrac{\bsym{R}_1^*}^\trans \bsym{B}' \hbsym{r}_1 \\
      &= \bsym{r}^\trans \bsym{B} \hbsym{r}_1
      - \rho\pbrac{a_1, L_1} \hbsym{r}_1^\trans \bsym{B} \hbsym{r}_1 \\
      &= \bsym{r}^\trans \bsym{B} \hbsym{r}_1
      - \rho\pbrac{a_1, L_1} d.
    \end{aligned}
  \end{equation}
  Therefore~(\ref{lemma3eq1}) also holds for $\lambda = 1$.
\end{proof}

\subsection{Proof of Theorem~\ref{theorem2}}\label{appendix_proof_of_theorem2}

\begin{proof}
  Exactly, $\hbsym{r}_3$ satisfies the equation
  \begin{equation}
    \bsym{n}_0 \cdot \hbsym{r}_3 - \rho\pbrac{a_3, L_3} d = 0.
  \end{equation}
  Subtract this from~(\ref{r3approx}) and we have the projection error
  \begin{equation}
    \begin{aligned}
    \abs{\hbsym{n}_0 \cdot \pbrac{\hbsym{r}_3^{(0)} - \hbsym{r}_3}} &= \frac{\gamma L_3 \abs{d} - \rho\pbrac{a_3, L_3} d}{\norm{\bsym{n}_0}}\\
    &= \frac{L_3 \abs{d}}{\norm{\bsym{B}^\trans \bsym{r}}} \abs{\gamma - \frac{\sqrt{1-a_3^2}}{\arccos{a_3}}}.
    \end{aligned}
  \end{equation}
  We find that
  \begin{equation}
    \min_\gamma \left\{\max_{a_3 \in [0, 1]} \left\{\abs{\gamma - \frac{\sqrt{1-a_3^2}}{\arccos{a_3}}}\right\}\right\} = \frac{1}{2} - \frac{1}{\pi}
  \end{equation}
  when taking $\gamma = \frac{1}{2} + \frac{1}{\pi}$, so we get
  \begin{equation}
    \label{esteq1}
    \abs{\hbsym{n}_0 \cdot \pbrac{\hbsym{r}_3^{(0)} - \hbsym{r}_3}} \le \frac{\pbrac{\pi - 2} L_3 \abs{d}}{2 \pi \norm{\bsym{B}^\trans \bsym{r}}}.
  \end{equation}
  
  Next we turn to find the bound of the approximation of $\bsym{r}_1$. The exact equation is
  \begin{equation}
    \begin{cases}
      \bsym{n}_e \cdot \hbsym{r}_1 = 0,\\
      \bsym{n}_0 \cdot \hbsym{r}_1 - \rho\pbrac{a_1, L_1} d = 0,
    \end{cases}
  \end{equation}
  where $\bsym{n}_e = \bsym{B} \hbsym{r}_3, \bsym{n}_0 = \bsym{B}^\trans \bsym{r}$. Since both $\hbsym{r}_1^{(0)}$ and $\hbsym{r}_1$ are elements in $\mathbb{S}^2$ and perpendicular with $\bsym{n}_e$, we assume
  \begin{equation}
    \hbsym{r}_1 = \exp\pbrac{\bsym{n}_e^\wedge \beta} \hbsym{r}_1^{(0)},
  \end{equation}
  and therefore,
  \begin{equation}
    \label{esteq2}
    \norm{\hbsym{r}_1^{(0)} - \hbsym{r}_1} \le \abs{\beta}.
  \end{equation}
  By following the pattern of proving~(\ref{esteq1}), we have the similar bounded projection error for $\hbsym{r}_1$
  \begin{equation}
    \label{esteq3}
    \abs{\hbsym{n}_0 \cdot \pbrac{\hbsym{r}_1^{(0)} - \hbsym{r}_1}} \le \frac{\pbrac{\pi - 2} L_1 \abs{d}}{2 \pi \norm{\bsym{B}^\trans \bsym{r}}}.
  \end{equation}
  According to Figure~{\ref{figestr1v}}, we have
  \begin{equation}
    \label{esteq4}
    \begin{aligned}
    \abs{\beta} = &\ \arccos\pbrac{\norm{\gamma L_1 \bsym{r}_0} - \frac{\pbrac{\pi - 2} L_1 \abs{d}}{2 \pi \norm{\bsym{B}^\trans \bsym{r}}}}\\
    &\ - \arccos\min\left\{1, \norm{\gamma L_1 \bsym{r}_0} + \frac{\pbrac{\pi - 2} L_1 \abs{d}}{2 \pi \norm{\bsym{B}^\trans \bsym{r}}}\right\}.
    \end{aligned}
  \end{equation}
  If we assume the distance of translation $\bsym{r}$ is larger than the length of the first section $L_1$, i.e., $\norm{\bsym{r}} > L_1$, then we have
  \begin{equation}
    \label{esteq5}
    \begin{aligned}
    \norm{\bsym{B}^\trans \bsym{r}}
    &= \sqrt{\norm{\bsym{r}}^2 d^2 + \pbrac{\cdot}^2 + \pbrac{\cdot}^2 + \pbrac{\cdot}^2}\\
    &\ge \norm{\bsym{r}}\abs{d} > L_1 \abs{d}.
    \end{aligned}
  \end{equation}
  Thus, combining~(\ref{esteq2}), (\ref{esteq4}) and~(\ref{esteq5}) we know that the error of $\hbsym{r}_1$ is also bounded, namely,
  \begin{equation}
    \norm{\hbsym{r}_1^{(0)} - \hbsym{r}_1} \le \arccos{\frac{2L_1\abs{d}}{\pi\norm{\bsym{B}^\trans \bsym{r}}}} - \arccos{\frac{L_1\abs{d}}{\norm{\bsym{B}^\trans \bsym{r}}}}.
  \end{equation}
\end{proof}

\subsection{Jacobian Matrix}\label{appendix_jacobian}

According to \citep{chirikjian2011stochastic}, the left Jacobian and the right Jacobian for $\mathsf{SO}(3)$ are computed as
\begin{equation}
  \begin{aligned}
    &\bsym{J}_l\pbrac{\bsym{\omega}} = \bsym{I}+
    \frac{1-\cos{\norm{\bsym{\omega}}}}{\norm{\bsym{\omega}}^2}\bsym{\omega}^\wedge
    +\frac{\norm{\bsym{\omega}}-\sin{\norm{\bsym{\omega}}}}{\norm{\bsym{\omega}}^3}\pbrac{\bsym{\omega}^\wedge}^2,\\
    &\bsym{J}_r\pbrac{\bsym{\omega}} = \bsym{J}_l^\trans\pbrac{\bsym{\omega}}.
  \end{aligned}
\end{equation}
Hence the right Jacobian for $\mathsf{SE}(3)$ is
\begin{equation}
  \bsym{J}\pbrac{\bsym{\xi}} = \pmat{\bsym{J}_r\pbrac{\bsym{\omega}} & 0_{3\times 3} \\
  \exp{\pbrac{-\bsym{\omega}^\wedge}} \pdiff{}{\bsym{\omega}}\pbrac{\bsym{J}_l\pbrac{\bsym{\omega}}\bsym{v}}
  & \bsym{J}_r\pbrac{\bsym{\omega}}}.
\end{equation}

\end{document}